\documentclass{article}

\usepackage{microtype}
\usepackage{graphicx}
\usepackage{booktabs} %
\usepackage{xspace}
\usepackage{dsfont}
\usepackage{graphicx}
\usepackage{wrapfig}
\usepackage{caption}
\usepackage{subcaption}
\usepackage{hyperref}
\usepackage{mathrsfs}

\usepackage{amsmath}
\usepackage{amssymb}
\usepackage{mathtools}
\usepackage{amsthm}

\usepackage[capitalize,noabbrev]{cleveref}
\usepackage[accepted]{icml2022}

\theoremstyle{plain}
\newtheorem{theorem}{Theorem}[section]
\newtheorem{proposition}[theorem]{Proposition}

\theoremstyle{definition}
\newtheorem{definition}[theorem]{Definition}

\theoremstyle{remark}
\newtheorem{remark}[theorem]{Remark}
\newtheoremstyle{PropositionNum}
        {\topsep}{\topsep}              %
        {\itshape}                      %
        {}                              %
        {\bfseries}                     %
        {.}                             %
        { }                             %
        {\thmname{#1}\thmnote{ \bfseries #3}}%
    \theoremstyle{PropositionNum}
    \newtheorem{propositionnum}{Proposition}

\usepackage[textsize=tiny]{todonotes}

\newcommand{\ba}[0]{\mathbf{a}}

\newcommand{\bs}[0]{\mathbf{s}}
\newcommand{\bg}[0]{\mathbf{g}}

\newcommand{\state}[0]{\bs}

\newcommand{\obs}[0]{\bs}
\newcommand{\goal}[0]{\bg}
\newcommand{\action}[0]{\ba}

\newcommand{\ouracronym}[0]{GCB\xspace}

\icmltitlerunning{Analogies in Goal-Conditioned Reinforcement Learning}

\begin{document}
\twocolumn[
\icmltitle{Bisimulation Makes Analogies in Goal-Conditioned Reinforcement Learning}

\icmlsetsymbol{equal}{*}

\begin{icmlauthorlist}
\icmlauthor{Philippe Hansen-Estruch}{ucb}
\icmlauthor{Amy Zhang}{ucb,fair}
\icmlauthor{Ashvin Nair}{ucb}
\icmlauthor{Patrick Yin}{ucb}
\icmlauthor{Sergey Levine}{ucb}
\end{icmlauthorlist}

\icmlaffiliation{ucb}{University of California, Berkeley}
\icmlaffiliation{fair}{Meta AI Research}

\icmlcorrespondingauthor{Philippe Hansen-Estruch}{hansenpmeche@berkeley.edu}
\icmlcorrespondingauthor{Amy Zhang}{amyzhang@fb.com}

\vskip 0.3in
]

\printAffiliationsAndNotice{}  %

\begin{abstract}
Building generalizable goal-conditioned agents from rich observations is a key to reinforcement learning (RL) solving real world problems. Traditionally in goal-conditioned RL, an agent is provided with the exact goal they intend to reach. However, it is often not realistic to know the configuration of the goal before performing a task. A more scalable framework would allow us to provide the agent with an example of an analogous task, and have the agent then infer what the goal should be for its current state. We propose a new form of state abstraction called \emph{goal-conditioned bisimulation} that captures \emph{functional equivariance}, allowing for the reuse of skills to achieve new goals. We learn this representation using a metric form of this abstraction, and show its ability to generalize to new goals in simulation manipulation tasks. Further, we prove that this learned representation is sufficient not only for goal-conditioned tasks, but is amenable to any downstream task described by a state-only reward function. Videos can be found at \url{https://sites.google.com/view/gc-bisimulation}.
\end{abstract}

\section{Introduction}
\label{sec:intro}
Goal-conditioned RL has the potential to train agents that can accomplish a variety of tasks when given a goal. However, the way the goal is represented in a goal-conditioned RL problem has a critical effect on how the resulting policy will interpret goals. For example, if a person is asked to accomplish the goal of cutting a vegetable, they can perform this task for any vegetable. The goal is represented in a way that is \emph{invariant} to irrelevant factors, and potentially even \emph{equivariant} to factors that vary between similar tasks, such as cutting long carrots versus short radishes. We use this equivariance to generalize to new problems and domains, by drawing associations to problems and domains we have experienced in the past.
If we can acquire representations that have these properties, then we can similarly specify goals for goal-conditioned policies in one setting, and have them carry out those goals in many diverse settings. This frees us from needing an exact goal image for achieving the same functional task in a new setting. As shown in \cref{fig:teaser}, given an agent that has learned to cut carrots, rather than needing to present the agent with an image of chopped up radishes, we can instead say something along the lines of, ``Do the same thing to the radish that you did to the carrot.''

\begin{figure}[t]
    \centering
    \includegraphics[width=.6\linewidth]{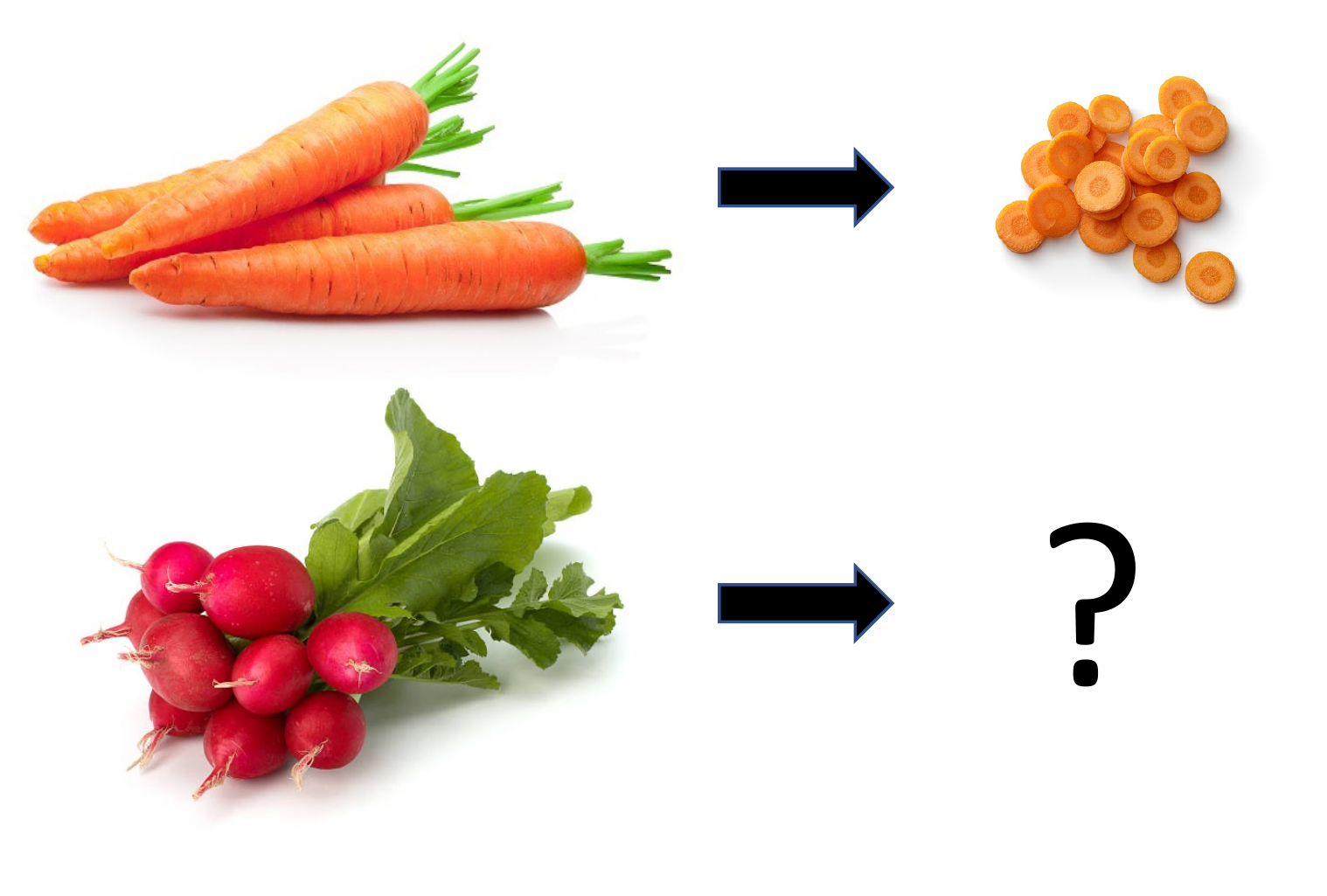}
    \vspace{-10pt}
    \caption{Analogous tasks of dicing carrots and radishes. Although the target objects are different, the skill required and functional difference between the initial state and goal images are similar. Our goal is to use this form of analogy to describe goals for new, unseen environments.}
    \label{fig:teaser}
\end{figure}

Prior work has proposed representation learning objectives that induce invariance to irrelevant factors through the use of domain knowledge about the relevant state features for a task~\cite{jonschowski2015priors, jonschkowski2017pve}, via data augmentations and contrastive learning~\cite{laskin_srinivas2020curl}, dynamics~\citep{watter2015embed,gelada2019deepmdp}, or proxy computer vision tasks (e.g., image reconstruction or segmentation)~\cite{lange2010deep, sax2018midlevel}.
In goal-conditioned problems however, which features are relevant or irrelevant depends not only on the current state but also the goal. While there may be uncontrollable aspects of an environment that can always be ignored, knowledge of the task and current state is necessary to determine what is relevant to complete the task. Rather than construct representations that ignore features that are never relevant, our goal is to construct functionally equivariant representations that only capture changes between state-goal pairs. 
These equivariant representations can then be applied to depict goals for new states as a form of analogy. A key question therefore is: how do we design a form of representation learning objective that can perform this type of analogy?

\begin{figure}[t]
    \centering
    \includegraphics[width=1\linewidth]{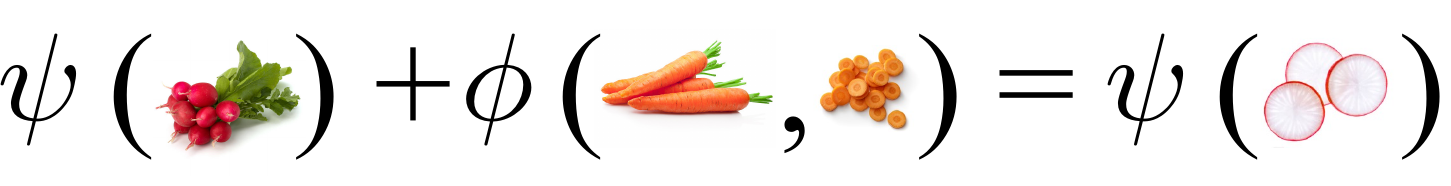}
    \vspace{-20pt}
    \caption{An example of abstractions with a compositional form.}
    \vspace{-10pt}
    \label{fig:teaser_analogy}
\end{figure}
In this work, we adapt a strict form of state abstraction called bisimulation~\citep{larsen1989bisim,ferns2004bisimulation} to address the goal-conditioned RL problem setting from rich observations. Rather than grouping equivalent states as bisimulation does, we group equivalent tasks --- allowing policies to leverage analogous tasks when attempting new ones. 
One can also define a policy-dependent distance metric form of our goal-conditioned bisimulation (\ouracronym) and use it to define a paired-state embedding space. 
We can then construct an objective for learning a state representation that can compose this task abstraction with states to depict new goals, capable of ``filling in the blank'', as shown in  \cref{fig:teaser_analogy}.  Such a representation is \emph{universal} --- it not only can be used for any goal-conditioned task, but is sufficient to train optimal policies for any state-only reward function.

Our contributions consist of the following: 
1) an objective that learns a functionally equivariant abstraction over goal-conditioned tasks, 2) a demonstration that this embedding space is capable of composing states with task abstractions to depict and solve for new goals, 3) a novel value bound that can be applied to any downstream task,
and 4) an evaluation on manipulation environments that shows improved performance on goal-conditioned tasks compared to other self-supervised representation learning methods.

\section{Related Work}
\label{sec:related_work}
Our approach learns general-purpose representations for goal-conditioned RL based on principles from state aggregation. Our approach is orthogonal and complementary to goal-conditioned methods: we are not proposing a new algorithm for goal-conditioned RL, but rather a representation learning approach that can be used with goal-conditioned methods (and can also be used for general RL problems). Therefore, in our experiments, we compare to other representation learning approaches and use the same goal-conditioned RL algorithm for all baselines. We review prior work on representation learning here and more extensive prior work on goal-conditioned RL in \cref{app:related_work}.

Work in \textbf{Representation learning for RL} can be decomposed into two categories: 1) state abstractions, which typically make use of task reward, and 2) self-supervised representation learning. As the scope of this body of work is very large, we focus on works exemplifying specific trends rather than a comprehensive review.
\citet{li2006stateabs} defined various forms of \textbf{state abstractions} for Markov decision processes (MDPs) that group states into clusters while preserving some property (e.g.\ the optimal value, all values, or all action values from each state). The strictest form, which generally preserves the most properties, is \textit{bisimulation}~\citep{larsen1989bisim}. Bisimulation only groups states that are indistinguishable w.r.t.\ reward sequences output given any action sequence tested. A related concept is bisimulation metrics~\citep{ferns2014bisim_metrics,ferns2011contbisim}, which measure how ``behaviorally similar'' states are. 
\citet{gelada2019deepmdp,zhang2021dbc,agarwal2021contrastive,castro2021mico} use behavior-based state abstractions to achieve better performance on control tasks by dropping irrelevant information.

\textbf{Self-supervised RL methods} rely on autoencoders~\citep{lange2010deep,Lange2012Autonomous,yarats2019sacae}, contrastive objectives~\citep{oord2018cpc,laskin_srinivas2020curl}, data augmentation~\citep{yarats2021drq}, and mutual information maximization~\citep{anand_unsupervised_2019,choi2021variationalgc} to learn representations for downstream control. 
Some work focuses on representation learning specifically for the goal-conditioned setting --- \citet{ghosh2018learning} uses goal-conditioned policies to extract a representation where Euclidean distances between states correspond to expected differences between actions to reach them. 
Similarly, \citet{yang2020plan2vec,tian2021modelbased} use functional distances to plan to reach goals.
However, none of these methods are capable of performing the types of analogies described in \cref{sec:intro}.

The idea of using \textbf{analogies} to improve learning is also not new~\citep{jaime1983analogy}, and much work has studied how compositionality and analogies arise naturally in learned representations.
\citet{mikolov2013word2vec} showed that latent language representations arising from skip-gram models were composable and produced analogies: e.g., $\phi(\text{``capital"}) + \phi(\text{``Germany"}) = \phi(\text{``Berlin"})$. Such analogies allow us to understand the latent space better and also demonstrate the structure of the latent space, which suggests it might be useful for downstream tasks. Similar analogies are also possible in visual domains~\cite{reed2015analogy,jayaraman-iccv2015}. 
In control, analogy representations enabled goal-directed grasping~\cite{devin2018grasp2vec} and composing plans~\cite{devin2019cpv}. Similar in spirit to this work, our method imposes a network architecture to learn analogies, but enforces it through goal-conditioned bisimulation, making it a more general method for RL.

\textbf{Metrics} have also been leveraged in RL for both MDPs and states for the goal of positive transfer~\citep{carroll2005tasksim,fernandez2006reuse,konidaris2012transfer}.
\citet{song2016mdpmetric} propose a metric between MDPs using the Kantorovich and Hausdorff metrics and use it for transferring value functions between source tasks and a target task. \citet{gelada2019deepmdp} present results that their learned representation bounds the bisimulation metric and \citet{zhang2021dbc,agarwal2021contrastive,castro2021mico} use on-policy versions of the bisimulation metric~\citep{castro20bisimulation} to learn an embedding space. Similar to \citet{zhang2021dbc}, we embed our metric in a representation space and show improved performance with downstream control, but with goal-conditioned tasks and the use of analogies rather than robustness to distractors.

\section{Preliminaries}
\label{sec:prelim}

We assume that the environment follows a \textbf{goal-conditioned Markov Decision Process} (GCMDP), which can be described with the tuple $\mathcal{M} = (\mathcal{S}, \mathcal{A}, \mathcal{P},  \mathcal{G}, \mathcal{R}, \gamma)$ where $\mathcal{S}$ is the state space, $\mathcal{A}$ the action space, $\mathcal{P}$ the dynamics model such that $\mathcal{P}(s' | s, a)$ dictates the probability of transitioning to state $s'$ from state $s$ after applying action $a$, $\mathcal{G} \subset \mathcal{S}$ is the goal space which is a subset of the state space, and $\mathcal{R}(s, a, g):=\mathds{1}(s'=g)$ is a sparse reward function which returns $1$ if $(s, a)$ transitions to goal $g$ and $0$ otherwise. We refer to a state-goal pair $(s,g)$ as a \emph{task}, drawing connections to the single task RL setting where tasks are defined by state-action reward functions.
While our primary concern is learning from images, we do not address the partial-observability problem explicitly: we follow prior work and use stacked pixel observations as the fully-observed system state $s$.

As noted in \cref{sec:related_work}, \textbf{bisimulation} is a form of state abstraction that groups states $\bs_i$ and $\bs_j$ that are ``behaviorally equivalent''~\citep{li2006stateabs}.
For any action sequence $\ba_{0:\infty}$, the probabilistic sequence of rewards from $\bs_i$ and $\bs_j$ are identical.
A more compact definition has a recursive form: two states are bisimilar if they share both the same immediate reward and equivalent distributions over the next bisimilar states~\citep{larsen1989bisim,Givan2003EquivalenceNA}. 
\vspace{-1mm}
\begin{definition}[Bisimulation Relations~\citep{Givan2003EquivalenceNA}]
Given an MDP $\mathcal{M}$, an equivalence relation $B$ between states is a bisimulation relation if, for all states $\state_i,\state_j\in\mathcal{S}$ that are equivalent under $B$ (denoted $\state_i\equiv_B\state_j$) the following conditions hold:
\begin{alignat}{2}
    \mathcal{R}(\state_i,\action)&\;=\;\mathcal{R}(\state_j,\action) 
    &&\quad \forall \action\in\mathcal{A}, \label{eq:bisim-discrete-r} \\
    \mathcal{P}(G|\state_i,\action)&\;=\;\mathcal{P}(G|\state_j,\action) 
    &&\quad \forall \action\in\mathcal{A}, \quad \forall G\in\mathcal{S}_B, \label{eq:bisim-discrete-p}
\end{alignat}
where $\mathcal{S}_B$ is the partition of $\mathcal{S}$ under the relation $B$ (the set of all groups $G$ of equivalent states), and $\mathcal{P}(G|\state,\action)=\sum_{\state'\in G}\mathcal{P}(\state'|\state,\action).$
\end{definition}
Bisimulation relates functionally equivalent states together in a way can be usefully applied to representation learning. 
Exact partitioning with bisimulation relations is generally impractical in continuous state spaces, as the relation is highly sensitive to infinitesimal changes in the reward function or dynamics. For this reason, \textbf{Bisimulation Metrics}~\citep{ferns2011contbisim,ferns2014bisim_metrics,castro20bisimulation} softens the concept of state partitions, and instead defines a pseudometric space $(\mathcal{S}, d)$, where a distance function $d:\mathcal{S}\times\mathcal{S}\mapsto\mathbb{R}_{\geq 0}$ measures the ``behavioral similarity'' between two states\footnote{Note that $d$ is a pseudometric, meaning the distance between two different states can be zero.}.

Computing bisimulation metrics is a harder class of problem than just finding an optimal policy due to the requirement that \cref{eq:bisim-discrete-r} and \cref{eq:bisim-discrete-p} hold true for $\forall \action\in\mathcal{A}$. In control problems, we typically only care about the optimal action $\action^*$ for any state. 
\citet{castro20bisimulation} define on-policy bisimulation, which we adopt to get the following on-policy metric:
\begin{align}
d(\obs_i,\obs_j)&:=\mathbb{E}_{\action\sim \pi}|\mathcal{R}(\obs_i,\pi^*(\obs_i))-\mathcal{R}(\obs_j,\pi^*(\obs_j))| \\ & + \gamma\cdot W_1(\mathcal{P}(\obs_i, \pi^*(\obs_i)),\mathcal{P}(\obs_j, \pi^*(\obs_j)); d). \nonumber
\end{align}
However, this definition relies on a single reward function and focuses on grouping equivalent states. In \cref{sec:gcb_def} we explore how to modify this metric for goal-conditioned tasks in order to group equivalent (or analogous) tasks.

\section{Functional Equivariance in an Embedding}
\label{sec:concepts}
We now build on the concepts introduced in \cref{sec:prelim} to step through a goal-conditioned metric in \cref{sec:gcb_def}.
Na\"ive lifting of the bisimulation metric into goal-conditioned MDPs gives us a metric over a \emph{paired state space}, since it depends on both the current state and goal. This construction prevents compositional generalization to novel state-goal pairs, as we will show in \cref{sec:exps}.
To obtain the type of compositional behavior we desire, we must leverage this back into a single state embedding. One form this can take is through simple arithmetic operations, where adding the task abstraction embedding to a single state embedding yields the desired goal. We will address how to construct an objective that gives rise to this property in \cref{sec:analogies}. First, we introduce the basic version of applying bisimulation to goal-conditioned problems.

\subsection{Defining a Goal-Conditioned Metric}
\label{sec:gcb_def}
Na\"ively, we can apply state abstractions and metrics designed for single task settings to goal-conditioned ones by redefining them over pairs of states --- a state-goal pair that defines a task. We walk through that process in this section.

Bisimulation is defined only for a single reward function, and therefore cannot be trivially applied to goal-conditioned settings. Furthermore, the central idea behind bisimulation is that it is an \emph{equivalence relation} that allows us to cluster states to treat them analogously. We propose to lift this idea to families of tasks by defining an equivalence relation \emph{over tasks},
giving us a functionally equivariant abstraction. 
Now that we have defined our environment in \cref{sec:prelim}, we can now define a form of goal-conditioned bisimulation that defines an abstraction over state-goal pairs using a goal-conditioned reward function:
\begin{definition}[Goal-conditioned Bisimulation Relations]
\label{def:gc_bisim}
Given an MDP $\mathcal{M}$, as described in section 3, we define an equivalence relation $B$ between states. This relation is a goal-conditioned bisimulation relation if, for all state-goal pairs $(\state_i,\goal_i),(\state_j,\goal_j)\in\mathcal{S}\bigcup\mathcal{G}$ that are equivalent under $B$ (denoted $\state_i\equiv_B\state_j$) the following conditions hold:
\begin{alignat}{2}
    \mathcal{R}(\state_i, \action, \goal_i)&\;=\;\mathcal{R}(\state_j, \action, \goal_j), &&\quad \forall \action\in\mathcal{A}, \nonumber \\
    \mathcal{P}(G|\state_i,\action)&\;=\;\mathcal{P}(G|\state_j,\action) 
    &&\quad \forall \action\in\mathcal{A}, \forall G\in\mathcal{S}_B, \nonumber
\vspace{-3pt}
\end{alignat}
\end{definition} 
Similar to in \cref{sec:prelim}, we can define an on-policy version of this relation, which gives rise to a paired-state metric:
\begin{equation}
\begin{aligned}
\label{eq:gcbm}
\small
    d_\pi&(\obs_i,\goal_i; \obs_j, \goal_j) = \\ &|\mathcal{R}(\state_i,\pi(\state_i,\goal_i), \goal_i) - \mathcal{R}(\state_j, \pi(\state_j,\goal_j), \goal_j)| \\ 
    &+ \gamma \mathcal{W}_2(\mathcal{P}(\state_i' |\state_i, \pi(\state_i,\goal_i), \mathcal{P}(\state_j' |\state_j, \pi(\state_j,\goal_j))
\end{aligned}
\end{equation}
where $\pi$ is the goal-conditioned policy and $\mathcal{W}_2$ is the Wasserstein distance. This distance metric captures functional equivalence across tasks, but is only defined for state-goal pairs and therefore cannot handle single states. We address this issue in the next section by using this metric to define a new state abstraction capable of forming analogies.

\subsection{Defining a State Abstraction to Form Analogies} 
\label{sec:analogies}
To generalize to novel state-goal combinations and make analogies across tasks, we need a single state representation represented by a state encoder $\psi$ that still maintains the properties present in the state-goal representation $\phi$. 
We take inspiration from works that have shown the arithmetic properties of embeddings~\citep{mikolov2013word2vec,devin2018grasp2vec} to enforce the same type of intuitive structure in the state encoder:
\begin{equation}
\label{eq:analogy}
    \psi(\goal_i) - \psi(\state_i) = \phi(\state_i,\goal_i), \quad \forall \state_i\in \mathcal{S},\forall \goal_i\in\mathcal{G}.
\end{equation}
From \cref{def:gc_bisim} we know that if there exists $\state_j\in\mathcal{S},\goal_j\in\mathcal{G}$ such that  $\phi(\state_i,\goal_i)=\phi(\state_j,\goal_j)$, then
$$\psi(\goal_j) - \psi(\state_j) := \psi(\goal_i) - \psi(\state_i),$$
forming an ``analogy'' of state-goal pairs $(\state_i,\goal_i)$ and $(\state_j,\goal_j)$.
An example of a form of analogy the representation constructed by $\psi$ should be capable of is shown in \cref{fig:teaser_analogy}. 
By learning a single state representation $\psi$, we can use knowledge about different state-goal pairs and their properties to generalize to new state-goal pairs at evaluation time using a simple form of arithmetic in latent space. This property frees us from the traditional goal-conditioned structure, which requires access to the target goal we want an agent to reach. \emph{Instead, an agent can reason about goals it has never seen before through examples of analogous tasks.}

\section{\ouracronym: Constructing Embedding Spaces}
\label{sec:method}
We now present an algorithm that combines representation learning of the two components presented in \cref{sec:concepts} and downstream control for image-based goal-conditioned problems. In principle, \ouracronym can be paired with any goal-conditioned RL method, including both online and offline algorithms. In the scope of our experiments, however, we focus on the offline setting to decouple from exploration difficulties. We first show how \ouracronym utilizes the metric proposed in \cref{sec:gcb_def} to construct a single state representation amenable to generalization through the use of analogies. 

\begin{figure}[t]
    \centering
    \includegraphics[width=.9\linewidth]{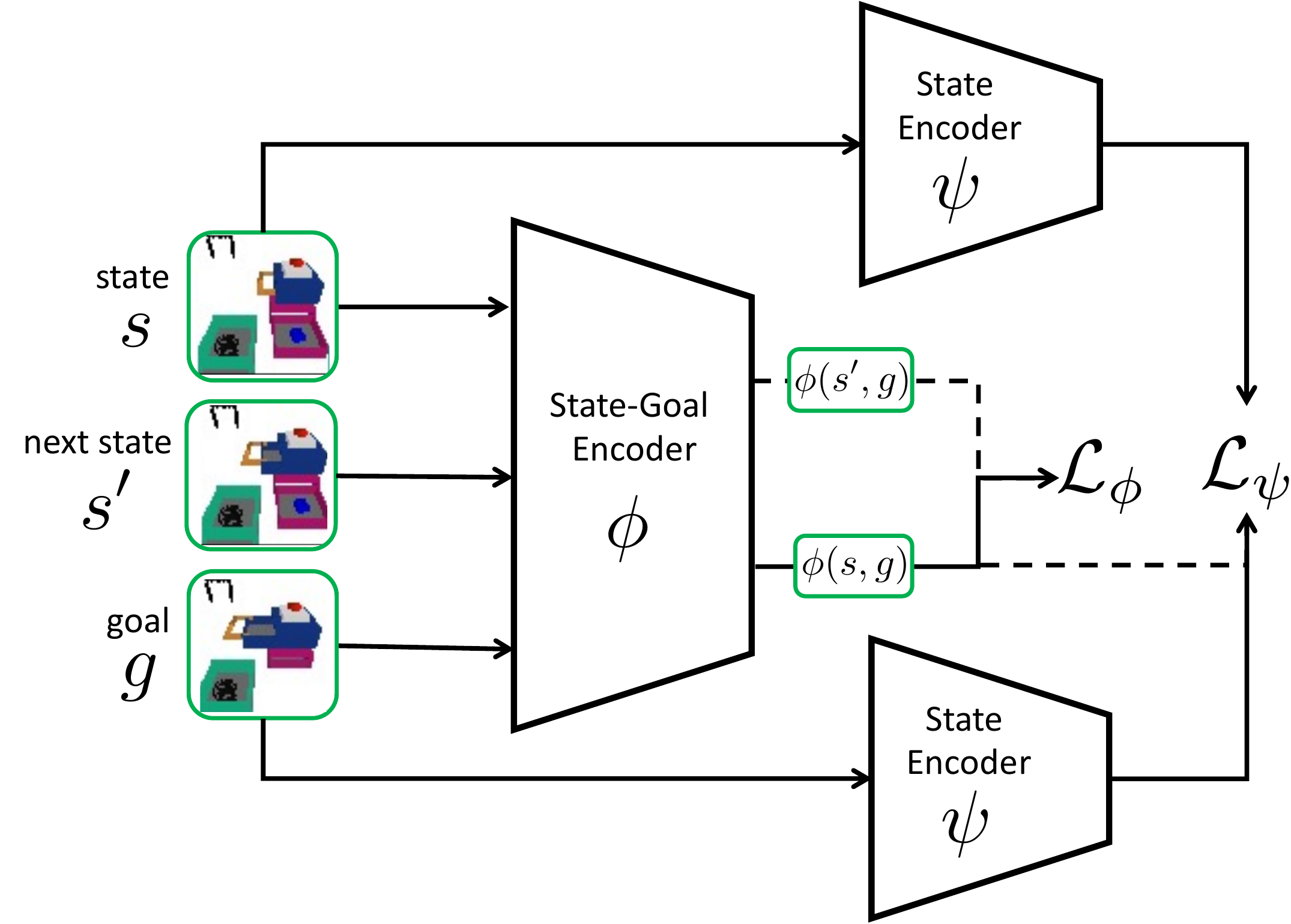}
    \vspace{-10pt}
    \caption{A flow diagram of the representation learning component of GCB. The dashed line represents stopped gradients and the state encoder $\psi$ is a Siamese network using shared weights.}
    \vspace{-10pt}
    \label{fig:main_alg}
\end{figure}

\subsection{Representation Learning Objectives}
Our end goal is to construct a single state representation space with the properties introduced in \cref{sec:analogies}. To do so, we must first learn an approximation of the policy-dependent GCB metric in \cref{eq:gcbm} end-to-end with our goal-conditioned policy $\pi$ during offline RL training. We construct an embedding space $\phi$ where $\ell_1$ distance between state-goal pairs corresponds to their distance under the GCB metric with the following objective:
\begin{align}\label{eq:phi_loss}
\vspace{-10pt}
\mathcal{L}_\phi &= \bigg(||\phi(s_i, g_i) - \phi(s_j, g_j)||_1  - ||r_i - r_j||_2  \\ 
&- \gamma ||\bar{\phi}(s'_i, g_i) - \bar{\phi}(s'_j, g_j)||_2\bigg)^2 \nonumber,
\vspace{-10pt}
\end{align} 
where $\bar{\cdot}$ denotes a stop-gradient. This objective is a direct instantiation of \cref{eq:gcbm} where, instead of learning a parameterized distance function $d_\pi$, we instead directly optimize for an embedding space $\phi$ with the above property. 
We prove that this update procedure converges to our desired metric in \cref{sec:theory}.

We can now use this $\phi$ embedding to define an objective for constructing a single state representation that is capable of performing simple forms of ``arithmetic'' as described in \cref{eq:analogy} in \cref{sec:analogies}:
\begin{equation}
\label{eq:psi_obj}
\vspace{-10pt}
\mathcal{L}_\psi = \bigg((\bar{\phi}(s_i, g_i) - \bar{\phi}(g_i, g_i)) - (\psi(g_i) - \psi(s_i))\bigg)^2.
\end{equation}
Note that \cref{eq:psi_obj} is a slightly different instantiation of \cref{eq:analogy} where we add $\phi(g_i,g_i)$ as a normalizing constant. We found this to perform better empirically.\footnote{$\phi(g, g)$ is the universal goal point in the paired state encoder space, it should map to the same constant point for all goals.} An ablation of this design decision can be found in \cref{app:ablations}. 
This objective will imbue the single state representation $\psi$ with the ability to form analogies across tasks.

In summary, \ouracronym learns two representations: $\phi(s, g)$, which generalizes across state-goal pairs, and $\psi(s)$, which is shaped in latent space by $\phi$ but can generalize across single states. We will show in \cref{sec:analogy_experiments} that $\phi$ can capture \emph{invariances} across tasks, which one can think of as only storing the ``delta'' between the initial state and goal, whereas $\psi$ is capable of a form of arithmetic that captures \emph{functional equivariance} across states and analogous tasks.

\begin{algorithm}[t]
   \caption{\ouracronym Training Algorithm} \label{alg:Goalbisim}
\begin{algorithmic}[1]
   \STATE \textbf{Given:} RL algorithm $\mathcal{A}$, replay buffer $\mathcal{D} = \{(s, a, s', r, g)_i\}_{i=1}^N$, learning rates $\alpha_\psi$, $\alpha_\phi$
  \STATE Initialize $\pi$, $\phi$, $\psi$
   \FOR{training iteration $k=1, 2, ...$} %
  \STATE $B_1 \leftarrow \{s_i, a_i, s_i', r_i, g_i\}_{i=1}^B \sim \mathcal{D}$\algorithmiccomment{Sample Batch}
  \STATE $B_2 \leftarrow \{s_j, a_j, s_j', r_j, g_j\}_{j=1}^B =$ permute $B_1$
   \STATE $\phi \leftarrow \phi - \alpha_{\phi} \nabla_{\phi}\mathcal{L}_{\phi}(B_1, B_2)$ 
   \STATE $\psi \leftarrow \psi - \alpha_{\psi} \nabla_\psi\mathcal{L}_\psi(B_1)$ 
   \STATE Update $\pi(\psi(s), \phi(s, g))$ with $\mathcal{A}$ on $B_1$
   \ENDFOR
\FOR{eval episode $m=1, 2, ...$} %
  \STATE Sample analogous tasks $(s, g)$ and $(s_a, g_a)$
  \STATE Rollout policy with $\pi(\psi(s), \phi(s_a, g_a))$
  \STATE Compute success of final achieved state against $g$
\ENDFOR
\end{algorithmic}
\end{algorithm}

\subsection{Combining GCB with Offline Downstream Control}
In this section, we describe a way to combine \ouracronym with goal-conditioned RL. Note that \ouracronym is primarily a representation learning method, and is therefore agnostic to the downstream goal-conditioned RL algorithm. 
In the standard goal-conditioned setting we learn a policy $\pi(s, g)$, which with our learned representations becomes $\pi(\psi(s), \psi(g))$.
We show in~\cref{sec:benchmark_experiments} that this enables better performance on goal-conditioned tasks.

However, \ouracronym is designed primarily to solve tasks specified by analogies, where at test time the task goal $g$ is unknown but instead specified by a separate state-goal pair $(s_a, g_a)$ that achieves an analogous outcome with respect to \emph{another} state $s$.
Here, we wish to learn a policy that can condition on standard state-goal pairs $\pi(\psi(s),\phi(s, g))$ and successfully evaluate on an analogous pair $\pi(\psi(s),\phi(s_a,g_a))$.
In \cref{sec:analogy_experiments}, we see how \ouracronym enables good performance when the policy is tasked by an analogous state-goal pair.

\subsection{Architecture \& Additional Implementation Details}
As shown in \cref{fig:main_alg}, \ouracronym uses two encoders: $\phi(s,g)$, which acts as the paired state/goal encoder, and $\psi(s)$, which acts as the state encoder.
Each is implemented as a convolutional neural network with six layers with 32 filters and map to latent spaces with dimensionality 256. 
We use Adam for optimization~\citep{kingma2014adam}.
As we focus on the offline RL setting in our experiments, we implement our method on top of implicit Q-learning (IQL)~\cite{kostrikov2021iql}, a recent offline RL algorithm. 
The representation and policy are trained concurrently.

The $\ell_1$ distance objective in \cref{eq:phi_loss} is enforced with random pairs of transitions. As is standard in contrastive learning, to generate random pairs, we sample a batch $(s, a, s', r, g)_i$ from the replay buffer and randomly permute the samples in the batch as $(s, a, s', r, g)_j$.

In~\cref{app:ablations} we present ablations to evaluate several important implementation decisions: (1) Adding a grounding goal point $\phi(g_i, g_i)$ as a normalizing constant. (2) The choice of $\ell_1$ distance versus $\ell_2$ distance for matching the \ouracronym metric. (3) Backpropagation of RL gradients to the encoder. (4) Adding a reward decoder model $\mathcal{R}(\phi(s,g), \phi(s',g))$ and (5) learned dynamics model as in DBC~\cite{zhang2021dbc}.
Additional implementation details and hyperparameters are in \cref{app:implementation_details}.

\cref{alg:Goalbisim} details the training of \ouracronym. 
The entire process is summarized as follows: 1) we sample a batch from the replay buffer, 2) randomly permute and match the batch and update $\phi$ according to \cref{eq:phi_loss}, 3) then update $\psi$ using $\phi$ according to \cref{eq:psi_obj}, and finally 4) update $\pi$ using the RL algorithm, where $\pi$ can be parameterized either as $\pi(\psi(s_i), \psi(g_i))$ or $\pi(\psi(s_i),\phi(s_a,g_a))$ depending on the goal specification setting.

\section{Value Bounds and Sufficiency Results}
\label{sec:theory}
In this section, we first show that our update procedure for the goal-conditioned bisimulation metric has a unique fixed point which is our desired metric. Further, we draw connections between value functions and this metric that allows us to provide transfer guarantees for a learned policy to new goals.
These results are simple extensions of existing results on bisimulation metrics~\citep{ferns2004bisimulation}. Finally, we also present a result on the universality of a representation learned with this metric that allows us to extend the above bounds to a larger class of downstream reward functions. 

We first show that our goal-conditioned metric definition has a unique fixed point.
\begin{proposition}
\label{prop:goalbisim_fixed_point}
\vspace{-2pt}
Let $\mathfrak{met}$ be the space of bounded pseudometrics on $\mathcal{S}$ and $\pi$ a goal-conditioned policy that is continuously improving. Define $\mathcal{F}:\mathfrak{met} \mapsto \mathfrak{met}$ by
\begin{align}
\mathcal{F}&(d,\pi)(\state_i,\goal_i; \state_j, \goal_j) = \\ 
&(1-c)|\mathcal{R}(\state_i, \pi(\state_i,\goal_i), \goal_i) - \mathcal{R}(\state_j, \pi(\state_j,\goal_j), \goal_j)| \nonumber \\ 
& + c W(d)(\mathcal{P}^{\pi}_{\state_i},\mathcal{P}^{\pi}_{\state_j}), \nonumber
\end{align}
where $c\in(0,1)$ is a metric discount factor.
Then $\mathcal{F}$ has a least fixed point $d^\pi$ which is a  goal-conditioned $\pi$-bisimulation metric. 
\end{proposition}
Proof in \cref{app:proofs}. This result shows that our update procedure in \cref{eq:phi_loss} will converge to our desired metric.
Next, we can show that, for any policy $\pi$, the \ouracronym metric upper bounds the value difference between any two tasks, as described by state-goal pairs. 
\begin{proposition}\label{prop:value_bound}
For any two state goal pairs $(\state_i,\goal_i), (\state_j,\goal_j)\in (\mathcal{S}, \mathcal{G})$ and a given policy $\pi$,
\begin{equation}
    |V^\pi(\state_i,\goal_i) - V^\pi(\state_j,\goal_j)| \leq d^\pi(\state_i,\goal_i; \state_j,\goal_j).
\end{equation}
\vspace{-20pt}
\label{prop:value_fn_bound}
\end{proposition}
Proof in \cref{app:proofs}.
For our learned policy $\pi$, this bounds the performance difference of it being applied to a new task denoted by state-goal pair $(s_j,g_j)$, given its performance on a known task $(s_i,g_i)$. Another way to interpret \cref{prop:value_bound} is as an intuitive result similar to Bellman's~\citep{Bellman:DynamicProgramming} policy improvement theorem and the extension in \citet{barreto_successor_2017}. \emph{If two tasks have a small distance in $d^\pi$, the policy $\pi$ will exhibit similar performance when deployed on them.}

We now show a novel result that the representations learned by $\psi$ and $\phi$ are sufficient for any downstream task, not just goal-reaching tasks.
\begin{proposition}\label{prop:downstream}
Assume that each state-action pair is visited infinitely often and the step-size parameters decay appropriately. Also assume that the goal space is the same as the state space, such that $\mathcal{G} := \mathcal{S}$. Learning with abstraction $\psi$ for any reward function that can be expressed as $\mathcal{R}:\mathcal{S} \mapsto \mathbb{R}$ converges to the optimal state-action value function in the original MDP.
\end{proposition}
Intuitively, \cref{prop:downstream} tells us that if our goal space is sufficiently rich, i.e., our agent learns to access every possible state in the MDP, then the information in $\psi$ and $\phi$ is sufficient to train an optimal policy for any task that can be described by a state-only reward function. 

Furthermore, we can extend these results to any downstream reward function. To do so, we first note that:
\begin{remark}\label{re:rew_fns}
For a discrete state space $\mathcal{S}$ with cardinality $|\mathcal{S}|$, any reward function $\mathcal{R}:\mathcal{S} \mapsto \mathbb{R}$ can be written as a tabular combination of states in the form $\mathcal{R}=\{\alpha_i s_i\}_{i=0}^{|\mathcal{S}|}$, where $\alpha_i\in\mathbb{R}, \forall i$, and therefore there also exists a representation of that state space (such as one-hots) for which the reward can be written as a linear combination of states. 
\end{remark}
Then, we can also bound the value difference between two states $\state_i,\state_j\in\mathcal{S}$ for a given policy $\pi$ and for any reward function as follows:
\begin{proposition}\label{prop:value_bound_anyr} 
For any reward function $R\in\mathcal{R}$ as defined above and a given policy $\pi$,
\begin{equation}
\vspace{-5pt}
    |V_R^\pi(\state_i) - V_R^\pi(\state_j)| \leq \sum_{k=0}^{|\mathcal{S}|} \alpha_k d^\pi(\state_i,\goal_k; \state_j,\goal_k),
\vspace{-5pt}
\end{equation}
where $V^\pi_R$ denotes the value function for a given policy $\pi$ and reward function $R$.
\end{proposition} %
Thus, we have shown that not only is our representation learning objective sufficient for any downstream task, but is structured in a way that is amenable for \emph{any} state-only reward function. One can view our method as a general, reward-agnostic representation learning scheme that exhibits nice properties for any downstream task.   
We now show the empirical capabilities of our method in \cref{sec:exps}.

\section{Experiments and Results} 
\label{sec:exps}

We evaluate and compare \ouracronym to other representation learning methods in the goal-conditioned setting. A central premise behind the design of \ouracronym is that it can capture functional equivariance over different tasks, as described by state-goal pairs.
Our experiments are therefore designed to answer the following questions: 1) Does \ouracronym truly capture these types of analogies?
2) Can \ouracronym use demonstration state-goal pairs to infer goals for novel states and successfully complete those tasks? 
3) Can \ouracronym also achieve strong performance on standard, offline goal-conditioned tasks?

\subsection{Baselines and Prior Methods}
We compare against representation learning methods following the major trends described in \cref{sec:related_work}. We evaluate two contrastive approaches: a temporal contrastive method (CPC)~\citep{oord2018cpc} and a data augmentation contrastive method (CURL)~\citep{laskin_srinivas2020curl}. 
We also compare to reconstruction-based methods that use a VAE and context-conditioned VAE (CC-VAE~\citep{nair19ccrig}).
Finally, we evaluate a pure data augmentation method, RAD~\citep{laskin2020rad}. All of the above representation learning methods produce single state representations, which we compare to our $\psi$ representation directly. 
As an additional ablation, we can also evaluate the paired state representation $\phi$ we learn as an intermediate objective, which corresponds to a na\"{i}ve goal-conditioned form of DBC~\citep{zhang2021dbc} which we denote by GC-DBC. 

We also include another baseline that leverage compositionality by arithmetic sums over vectors, compositional plan vectors (CPV, \cite{devin2019cpv}). However, CPVs are evaluated in a one-shot imitation learning setting which assumes access to paired reference trajectories. CPV uses an architecture $\pi(s_t, v_t)$ where $v_t = g(o_0, o_T) - g(o_0, o_t)$, where $(o_0, o_t, o_T)$ are drawn from the reference trajectory in one-shot imitation learning to train $\pi(a_t|o_t)$. Our method does not have access to paired reference trajectories. Instead, we train the CPV architecture $\pi(s_t, v_t)$ with $(o_0, o_t, o_T)$ being drawn from the current trajectory. This allows us to learn a compositional latent space for RL. 

\begin{table*}[t]
    \centering
    \begin{tabular}{c|c|c|c|c|c|c}
        Method & \texttt{Drawer} &  \texttt{Drawer+VD} & \texttt{BD} & \texttt{BD+VD} & \texttt{Analogy} & \texttt{Analogy+VD} \\
        \hline
        GCB (ours) & $\mathbf{0.533 \pm 0.037}$ & $\mathbf{0.448 \pm 0.033}$ & $0.440 \pm 0.026$ & $\mathbf{0.322 \pm 0.021}$ & $\mathbf{0.403 \pm 0.041}$ & $\mathbf{0.303 \pm 0.027}$\\
        CPV & $0.400 \pm 0.031$ & $0.209 \pm 0.013$ & $0.228 \pm 0.024$& $0.198 \pm 0.013$  & $0.176 \pm 0.011$& $0.121 \pm 0.009$\\
        GC-DBC & $0.269 \pm 0.025$ & $0.261 \pm 0.031$ &$0.358 \pm 0.023$ & $0.255 \pm 0.034$  & NA & NA\\
        CCVAE & $0.443 \pm 0.029$  & $0.231 \pm 0.017$ & $0.452 \pm 0.039$& $0.225 \pm 0.028$  & $0.234 \pm 0.012$& $0.095 \pm 0.009$\\
        VAE & $0.463 \pm 0.023$ & $0.261 \pm 0.018$ & $\mathbf{0.522 \pm 0.042}$& $0.115 \pm 0.019$  & $0.280 \pm 0.024$& $0.112 \pm 0.006$\\
        CPC & $0.305 \pm 0.020$ & $0.207 \pm 0.014$ & $0.363 \pm 0.042$& $0.154 \pm 0.023$  & $0.148 \pm 0.026$& $0.145 \pm 0.029$\\
        CURL & $0.259 \pm 0.021$ & $0.283 \pm 0.030$ & $0.382 \pm 0.038$& $0.198 \pm 0.022$  & $0.162 \pm 0.021$& $0.123 \pm 0.025$\\
        RAD & $0.245 \pm 0.019$ & $0.207 \pm 0.010$ & $0.159 \pm 0.036$& $0.108 \pm 0.018$ & $0.157 \pm 0.028$& $0.110 \pm 0.026$\\
        Pixel & $0.483 \pm 0.051$ & $0.256 \pm 0.027$ & $0.117 \pm 0.016$& $0.136 \pm 0.027$ & $0.132 \pm 0.025$& $0.121 \pm 0.033$\\
    \end{tabular}
    \caption{Final avg. success rates with stderr over 5 seeds. \texttt{BD} refers to the \texttt{Button and Drawer} env. and \texttt{VD} refers to Video Distractors.}
    \label{tab:results}
\end{table*}

\subsection{Simulation Manipulation Experiments}
We evaluate \ouracronym on complex manipulation tasks in a PyBullet simulated environment \citep{coumans2021} consisting of randomly generated workspaces, as shown in \cref{fig:pybullet_example}. 
We focus on the offline RL setting to decouple the exploration difficulties of this environment from the learning difficulties on various representations. Each initial state in this environment contains a randomized variety of objects and a randomly positioned drawer assembly, providing a variety of possible goals and distractors to evaluate functional equivariance and generalization.

\begin{wrapfigure}[10]{r}{0.33\linewidth}
  \centering
  \vspace{-20pt}
  \includegraphics[width=1\linewidth]{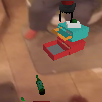}
  \vspace{-20pt}
  \caption{\texttt{Button and Drawer} env with video distractors.}
  \vspace{-5pt}
  \label{fig:pybullet_example}
\end{wrapfigure}

In the \texttt{Drawer} environment, a robot must learn to position a drawer in a target goal location, while in the other \texttt{Button and Drawer} environment, which consists of two stacked drawers with a button on top, the task of the robot is either to position the top drawer correctly or press the button which opens or closes a bottom drawer compartment. More details of the environment are in Appendix \ref{app:env_details}.
We create a distribution over tasks by randomizing the position, color, and orientation of the drawers. Additionally, we include between 0 and 4 distractor objects from a set of 84 object geometries, and in some experiments we also overlay distracting real video in the background~\citep{zhang2018natrl} from the Kinetics dataset~\citep{kineticsdataset}. %
To collect the offline replay buffer, we use a noisy expert policy $\pi_{demo}(a_i | s, g) = \pi^*(a_i | s, g) + \mathcal{N}(0, 0.3)$ where $ -1 \leq a_i \leq 1 \ \forall \ i$.  50K transitions are collected for the training set for the following offline RL experiments, where the demonstrated policy reaches the goal in roughly 80\% of attempted episodes. One epoch defines one full pass through the dataset. These environments were designed to contain several types of variation to showcase the different types of generalization \ouracronym is capable of.

 \begin{figure}[t]
  \centering
  \includegraphics[width=.99\linewidth]{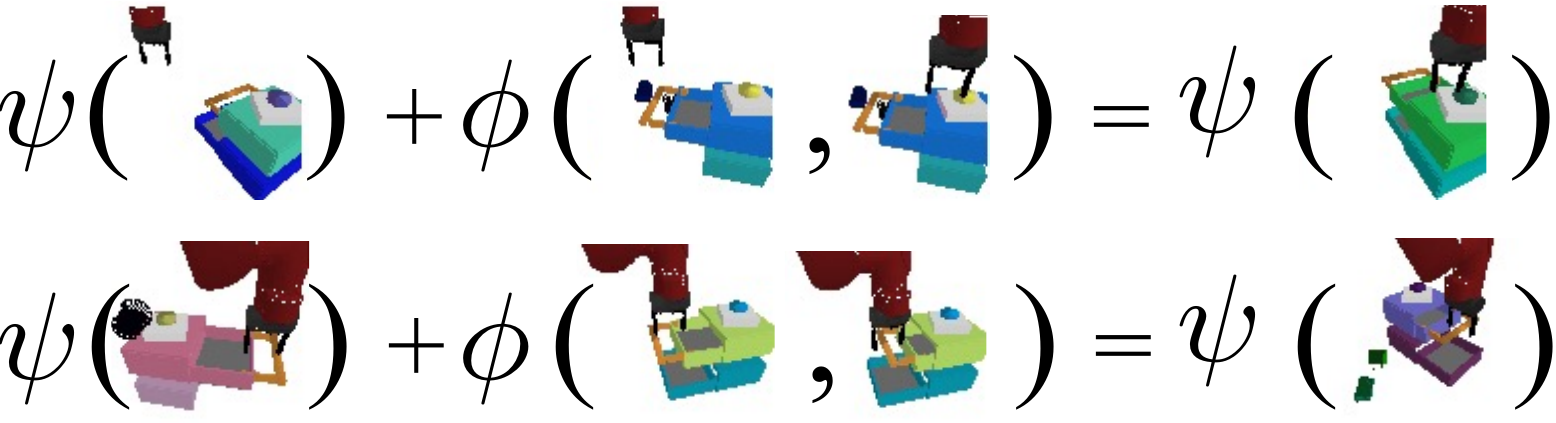}
 \vspace{-10pt}
  \caption{Two examples of \textbf{analogy arithmetic.} The rightmost image is the nearest neighbor in a test set of the composed representation in $\psi$ space. Top: RHS has the finger pushing the button, just like in the goal argument to $\phi$, but in a different position and orientation, indicating that the representation is equivariant to pose but captures the functional relevance of pushing the button.  Bottom: RHS has the drawer being closed, again matching function with the goal argument to $\phi$ while exhibiting invariance to color and equivariance to pose.}
  \label{fig:analogy_arith}
\end{figure}

\subsubsection{\ouracronym Makes Analogies}
\label{sec:analogy_experiments}
In this section, we probe our learned representations $\phi$ and $\psi$ to determine if they have the properties they were designed to have. In \cref{fig:phi_nn} we sample random tasks, encode them with $\phi$, and sample their nearest neighbors from a small subset of our offline dataset. According to \cref{def:gc_bisim}, these tasks being near each other in $\phi$ space means they are functionally similar. In comparing the sampled state-goal pairs, we see that \ouracronym successfully captures \emph{invariance}: features that are irrelevant to the task change across the sample and nearest neighbor in $\phi$ space, such as the color of the drawer and what distractor objects are present in the background. We show a similar result in \cref{fig:tsne} but with the space of $\phi$ shown in a t-SNE plot. We again see that randomly sampled points in this space exhibit similar semantic tasks. For \cref{fig:analogy_arith}, \cref{fig:phi_nn}, \cref{fig:tsne}, and \cref{fig:analogy_results} the images shown are the ones given directly to the encoders $\psi$ and $\phi$.
This visualization shows how the $\phi$ representation groups state-goal pairs and gives justification that \ouracronym can make meaningful analogies. In our next visualization, we will show that $\psi$ can exhibit \emph{functional equivariance.}

\begin{figure}[h]
  \centering
  \vspace{-5pt}
  \includegraphics[width=1\linewidth]{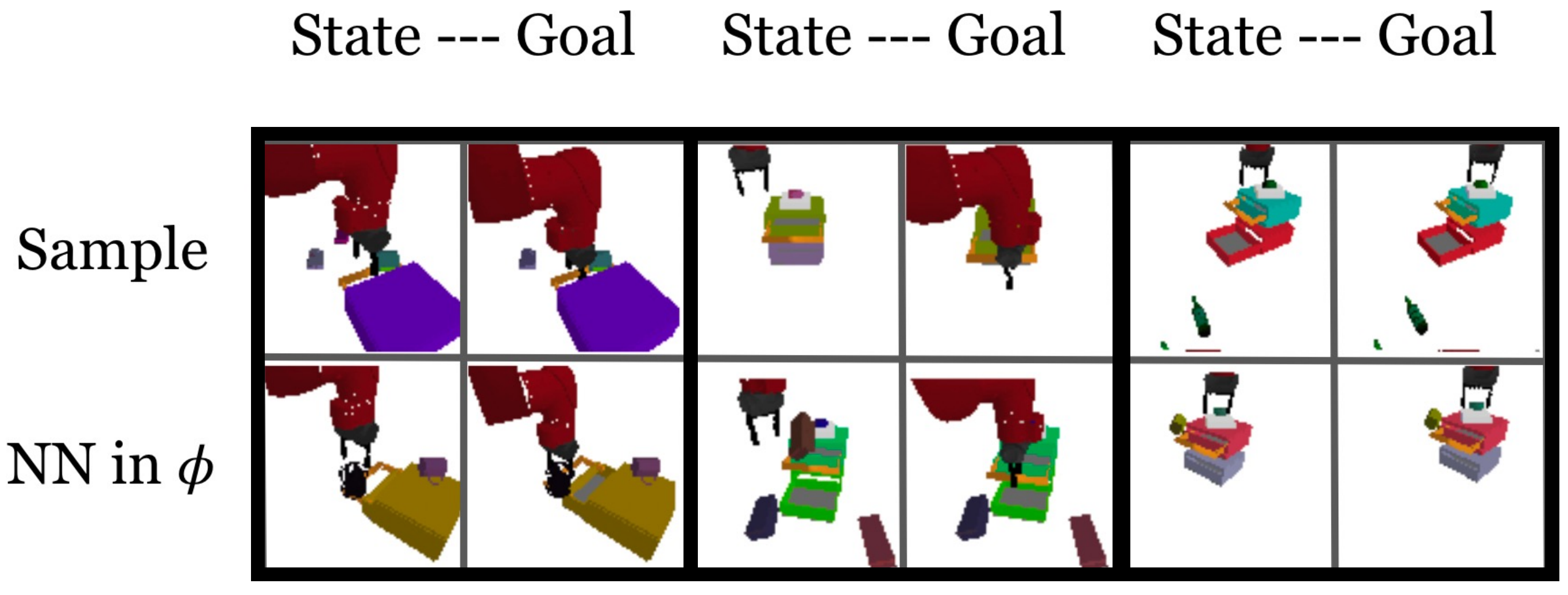}
  \vspace{-20pt}
  \caption{Visualization of nearest neighbors in $\phi$. We sample state-goal pairs and find their NN in $\phi$ space. Features irrelevant to the task have changed, such as drawer color and distractor objects, but task-specific components are invariant, such as the relative positioning of robot to target object and semantics of the task.} 
  \label{fig:phi_nn}
\end{figure}

\begin{figure}[h]
  \centering
  \vspace{-5pt}
  \includegraphics[width=1\linewidth]{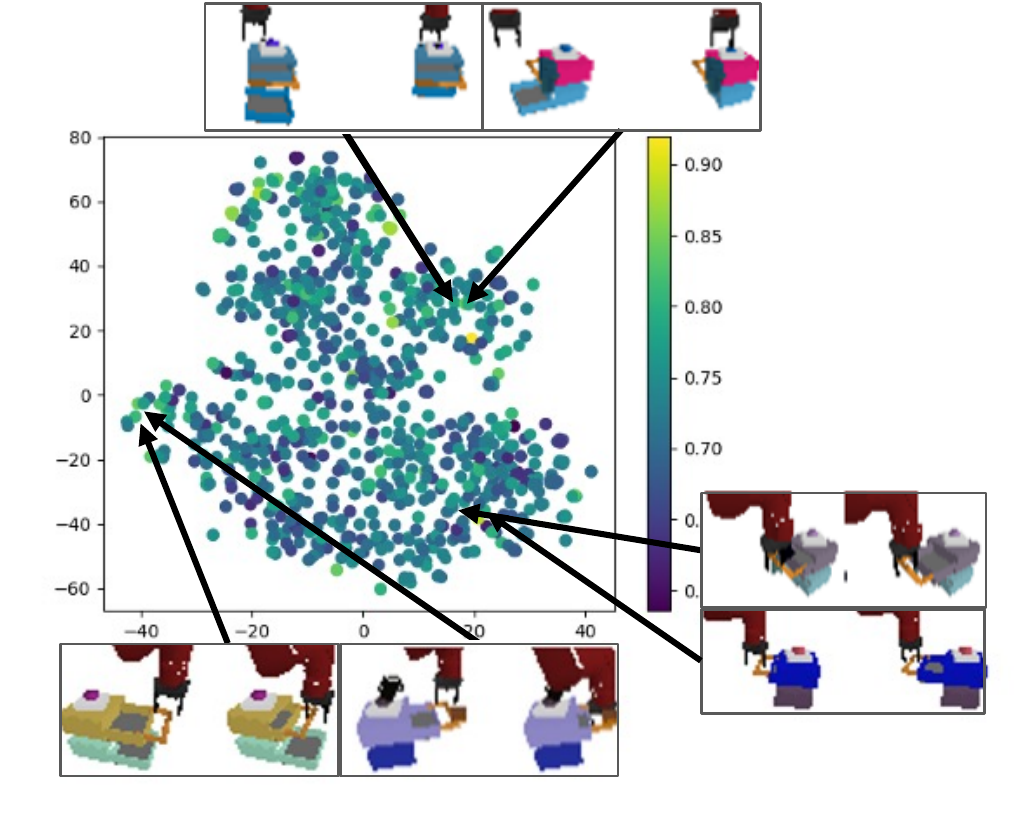}
  \caption{Visualization of nearest neighbors in $\phi$ as a t-SNE plot. Color corresponds to the predicted value output by the goal-conditioned critic.} 
  \label{fig:tsne}
\end{figure}

While \cref{fig:phi_nn} and \cref{fig:tsne} show that our $\phi$ learning objective successfully groups equivalent tasks, we must now show that $\psi$ is capable of \emph{functional equivariance}, where pertinent factors of variation carry over across tasks. To test this, we compose example state-goal pairs with new states to describe new goals, in the manner shown in \cref{fig:teaser_analogy}. 
In \cref{fig:analogy_arith} we sample random states $s\in\mathcal{S}$ and an analogous state/goal ($s_a, g_a$) pair that is functionally similar but has a potentially different color drawer, background distractors, orientation, and/or position. We encode the state and analogous state/goal, and add them to produce $\psi(s) + \phi(s_a, g_a)$. We then sample the nearest neighbor of this point in $\psi$ space to find the closest representation of the desired goal. We find that $\psi$ is able to take the analogous task and produce a reasonable goal for the current environment. If the analogous problem results in pressing the button, $\psi$ then selects a similar state where the button is pressed (\cref{fig:analogy_arith}, top), \emph{maintaining the same orientation as the current initial state}.  If it results in closing the drawer, the goal produced by $\psi$ also corresponds to a state where the drawer is being closed (\cref{fig:analogy_arith}, bottom).
This indicates that the space is able to pick up on analogies, by representing tasks in an equivariant way such that applying the representation of a task from one initial state to another results in an analogous transformation to that state\footnote{We do not see the exact goal that matches the given initial state because it is not part of the dataset we sample from.}.

As discussed in \cref{sec:analogies}, we can use analogies to describe goals for tasks where we do not have access to the true intended goal. \cref{fig:analogy_arith} constitutes a qualitative check that this is likely true, and we design the following experiment to quantitatively measure this ability. We collect an evaluation set where each initial state is paired with an analogous state-goal pair $(s_a, g_a)$, where factors that the agent should be invariant or functionally equivariant to vary, similar to \cref{fig:analogy_arith}. 
The policy is trained with $\pi(\psi(s), \phi(s, g))$, but is evaluated with $\pi(\psi(s), \phi(s_a, g_a))$. We focus on the \texttt{Button and Drawer} environment as it has a richer space of goals --- the policy has to use information from $s_a, g_a$ to infer the correct goal. Success in this environment requires representations that encode important task information across visually different tasks.

As quantitative evidence of \ouracronym's ability to use analogies, we see in \cref{tab:results} that
 \ouracronym remains the best performer, achieving 40\% success on average. 
This shows that \ouracronym can map analogous problems to the current environment, successfully reaching goals that it was never given goal images of --- rather only an analogous task as an example. 
In the next section, we see if these properties also help \ouracronym on standard goal-conditioned RL tasks.

\subsubsection{Comparisons on goal-reaching}
\label{sec:benchmark_experiments}
We have shown that \ouracronym exhibits invariance and functional equivariance across tasks, which leads to superior performance in a task setting where the agent is given an analogous task as a goal descriptor rather than the specified goal image at evaluation time. Now we  ask, does \ouracronym also lead to improved performance on standard goal-reaching tasks? 
We evaluate in the standard goal-conditioned setting on new, unseen environments where we provide the policy $\pi(s,g)$ the intended true goal image $g$. 

In the standard goal setting without video distractors, GCB is able to outperform all other representation methods on the \texttt{Drawer} benchmark achieving 50\%+ success, as seen in \cref{tab:results}. In \texttt{Button and Drawer}, GCB maintains strong performance, beating all representations except for VAE and CCVAE, where it remains competitive. When real video distractors are added, GCB performs the best in both tasks. 
The performance of GC-DBC shows the necessity of the compositional single state representation, $\psi$ --- without it, we are not able to achieve as strong performance on new, unseen environments, as was noted in \cref{sec:concepts}.

\section{Discussion}
This paper introduces goal-conditioned bisimulation (\ouracronym), a representation learning method for goal-conditioned RL problems that captures functional equivariance over a family of goal-conditioned tasks. We have shown that, at evaluation time, an agent can use representations of analogous tasks
to achieve unseen goals and even outperform existing representation learning methods in the standard goal-conditioned setting. Furthermore, our method is theoretically backed by an equivalence relation that defines what ``functional similarity'' means. The representations learned by our method are structured,
in that $\ell_1$ distance in the space bounds the value difference of different states and tasks. This holds true not only for goal-conditioned problems, but all downstream tasks describable by a state-only reward function. 

While we show improved success rates in goal-conditioned tasks and improved generalization capability on unseen, in-distribution environments, it is possible that we can broaden the definition of functional equivalence beyond that of goal-conditioned bisimulation. As an example, one limitation of \ouracronym is that the number of steps the policy needs to complete the task is a determining factor in how close two tasks are. While it is not entirely clear how we can broaden this equivalence relation while maintaining some semantic meaning of \emph{function}, the potential of such an abstraction is plain. We can define equivalences across tasks in different domains and action spaces, allowing agents to learn from third-person demonstrations and unlocking the potential of large, off-domain datasets.

\section*{Acknowledgements}
This research was supported by a Meta AI-BAIR Commons project, AN is supported by the Army Research Office under W911NF-21-1-009, and the Office of Naval Research, with compute support from Google Cloud, Berkeley Research Computing, Meta, and Azure.

\bibliographystyle{plainnat_ours}
\bibliography{ref}

\newpage
\appendix
\onecolumn
\section{Extra Visualizations}
We add extra Nearest Neighbors visualizations for our arithmetic space. These follow the same logic as shown in Figure \ref{fig:analogy_arith}. The analogies are sampled by randomizing color, distractor objects, video distractors (if present), and translating and rotating the drawer. We provide more details about how we sample these analogies in \cref{sec:analogy_experiments}.
\begin{figure}[h]
  \centering
  \includegraphics[width=0.48\linewidth]{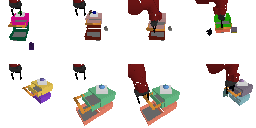} \\
  \includegraphics[width=0.48\linewidth]{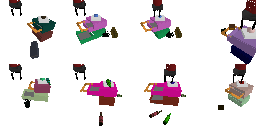} \\
  \includegraphics[width=0.48\linewidth]{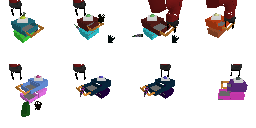} \\
  \includegraphics[width=0.48\linewidth]{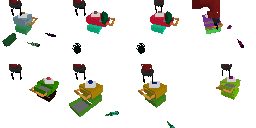}
  \caption{More examples of \textbf{analogy arithmetic.} The leftmost image is a sampled start state. The next two images are a sampled analogy start and goal. The rightmost image is the nearest neighbor in a test set of the composed representation in $\psi$ space. Each row is a new sample, the second to last row is an example of a failure.}
  \vspace{-20pt}
  \label{fig:analogy_results}
\end{figure}

\section{Extended Related Work}
\label{app:related_work}

\paragraph{Goal-Conditioned RL.} 
Goal-conditioned RL learns a policy that can generalize to many reward functions, indexed by goals, rather than a single reward function~\cite{kaelbling1993goals, schaul2015uva, andrychowicz2017her, nachum2018hiro, eysenbach2020rewriting}.
For environments with rich, high-dimensional observations such as images, methods have proposed to use existing generative models for encoding observations and goals~\cite{nair2018rig, nair19ccrig, khazatsky2021val}. Using reconstructive generative models for the input representation compresses the state space by encoding appearance information, but does not relate to functionality.
Goal-conditioned RL has been used to learn general policies in the real world on large, diverse datasets~\cite{chebotar2021actionable}. Thus making goal-conditioned RL more sample efficient and performant is a practical contribution towards real-world reinforcement learning. Non-parametric methods for exploration and planning with images have been considered~\cite{wadefarley2019discern, nasiriany2019leap, wang2020ropes, liu2020ropes, shah2021rapid} and could also benefit from improved representation learning.
\citet{han2021learning} address the generalization problem of goal-conditioned policies to novel observations at test time by using an alignment objective for data collection across environments. Our method instead naturally finds alignment through an equivalence relation. 

\section{Proofs of Theoretical Results}
\label{app:proofs}
\begin{theorem}\label{thm:pi_castro}
 Define $\mathcal{F}^{\pi}:\mathcal{M}\rightarrow\mathcal{M}$ by
  $\mathcal{F}^{\pi}(d)(s, t) = |\mathcal{R}^{\pi}_s - \mathcal{R}^{\pi}_{t}| + \gamma
  \mathcal{W}_1(d)(\mathcal{P}^{\pi}_s, \mathcal{P}^{\pi}_{t})$, 
  then $\mathcal{F}^{\pi}$ has a least fixed point $d^{\pi}$, and
  $d^{\pi}$ is a $\pi$-bisimulation metric.
\end{theorem}
\begin{propositionnum}[\ref{prop:goalbisim_fixed_point}]
\begin{align*}
\mathcal{F}(d,\pi)(\state_i,\goal_i; \state_j, \goal_j) = 
(1-c)|\mathcal{R}(\state_i, \pi(\state_i,\goal_i), \goal_i) - \mathcal{R}(\state_j, \pi(\state_j,\goal_j), \goal_j)| 
+ c W(d)(\mathcal{P}^{\pi}_{\state_i},\mathcal{P}^{\pi}_{\state_j}). 
\end{align*}
Then $\mathcal{F}$ has a least fixed point $d^\pi$ which is a  goal-conditioned $\pi$-bisimulation metric.
\end{propositionnum}
\begin{proof}
We start from the result in Theorem 2 of \citet{castro20bisimulation}, also included as \cref{thm:pi_castro} here. We can reconstruct our GCMDP as an equivalent, standard super-MDP $\mathcal{M}'$ by concatenating the state and goal spaces to construct a new state space: $\mathcal{S}':= \mathcal{S}\times\mathcal{G}$ with new reward function $\mathcal{R}':=\mathds{1}(s'=(g,g))$. We can then rewrite $\mathcal{F}$ as:
\begin{equation*}
    \mathcal{F}_{\mathcal{M}'}(d,\pi)(\state'_i, \state'_j) = 
(1-c)|\mathcal{R}'(\state'_i, \pi(\state'_i) - \mathcal{R}'(\state'_j, \pi(\state'_j))| 
+ c W(d)(\mathcal{P}^{\pi}_{\state'_i},\mathcal{P}^{\pi}_{\state'_j}). 
\end{equation*}
It should be clear that this is equivalent to the $\mathcal{F}$ defined for the GCMDP as the reward functions are identical and the dynamics do not change, since $\mathcal{P}$ does not depend on the goal. We can then apply \cref{thm:pi_castro} to this super-MDP to show that $\mathcal{F}_{\mathcal{M}'}$ has a unique fixed point which is the $\pi$-bisimulation metric, proving that the equivalent $\mathcal{F}$ also has the same unique fixed point.
\end{proof}

\begin{proposition}
	For any $\pi \in \mathscr{P}(\mathcal{A})^{\mathcal{S}}$ and states $x,y \in \mathcal{S}$, we have $|V^\pi(x) - V^\pi(y)| \leq d^\pi(x, y)$.
  \label{prop:mico_valueFunctionBound}
\end{proposition}

\begin{propositionnum}[\ref{prop:value_fn_bound}]
For any two state goal pairs $(\state_i,\goal_i), (\state_j,\goal_j)\in (\mathcal{S}, \mathcal{G})$,
\begin{equation*}
    |V^\pi(\state_i,\goal_i) - V^\pi(\state_j,\goal_j)| \leq d^\pi(\state_i,\goal_i; \state_j,\goal_j).
\end{equation*}
\end{propositionnum}
\begin{proof}
We start from the result in Proposition 4.8 in \citet{castro2021mico}, also reproduced here as \cref{prop:mico_valueFunctionBound}. 
We again reconstruct our GCMDP as an equivalent, standard super-MDP $\mathcal{M}'$ by concatenating the state and goal spaces to construct a new state space: $\mathcal{S}':= \mathcal{S}\times\mathcal{G}$ with new reward function $\mathcal{R}':=\mathds{1}(s'=(g,g))$. We can now apply \cref{prop:mico_valueFunctionBound} to this super-MDP, showing the above bound to be true.
\end{proof}

\begin{propositionnum}[\ref{prop:downstream}]
If we assume that the goal space is the same as the state space, i.e. $\mathcal{G} := \mathcal{S}$, then \ouracronym learns a representation sufficient to learn the optimal policy for any downstream reward function that can be expressed as $\mathcal{R}:\mathcal{S} \mapsto \mathbb{R}$.
\end{propositionnum}
\begin{proof}
As noted in \cref{re:rew_fns}, a state-only reward function can be written as a set of all states and their rewards,  $\mathcal{R}=\{\alpha_i s_i\}_{i=0}^{|\mathcal{S}|}$, where $\alpha_i\in\mathbb{R}, \forall i$. The optimal policy $\pi^*$ for any reward function $\mathcal{R}$ can then be expressed as a series of state-reaching policies, which will be within the set of goal-conditioned policies where $\mathcal{G}=\mathcal{S}$.
\end{proof}
\begin{propositionnum}[\ref{prop:value_bound_anyr}]
For any reward function $R\in\mathcal{R}$ as defined above and a given policy $\pi$,
\begin{equation*}
    |V_R^\pi(\state_i) - V_R^\pi(\state_j)| \leq \sum_{k=0}^{|\mathcal{S}|} \alpha_k d^\pi(\state_i,\goal_k; \state_j,\goal_k).
\end{equation*}
\end{propositionnum}
\begin{proof}
Without loss of generality, we assume $\alpha_i\geq 0,\forall i$.
Starting from \cref{prop:value_fn_bound} we have that $\forall g\in\mathcal{S},$ 
\begin{equation}\label{eq:apply_value_fn}
    |V^\pi(\state_i,\goal) - V^\pi(\state_j,\goal)| \leq d^\pi(\state_i,\goal; \state_j,\goal).
\end{equation}
We can express $V^\pi_R(\state_i)$ as a sum over goal-conditioned values:
$$V_R^\pi(\state_i)=\sum_{k=0}^{|\mathcal{S}|}\alpha_k V^\pi(\state_i,\goal_k).$$
We can plug this into the LHS:
$$|V_R^\pi(\state_i) - V_R^\pi(\state_j)| = \bigg|\sum_{k=0}^{|\mathcal{S}|}\alpha_k V^\pi(\state_i,\goal_k) - \sum_{k=0}^{|\mathcal{S}|}\alpha_k V^\pi(\state_j,\goal_k)\bigg|.$$
Combine terms into one summation,
$$|V_R^\pi(\state_i) - V_R^\pi(\state_j)| = \bigg|\sum_{k=0}^{|\mathcal{S}|}\alpha_k \big( V^\pi(\state_i,\goal_k) -  V^\pi(\state_j,\goal_k)\big)\bigg|.$$
Then use the triangle inequality to move the absolute value inside the sum,
$$|V_R^\pi(\state_i) - V_R^\pi(\state_j)| \leq \sum_{k=0}^{|\mathcal{S}|}\alpha_k \big| V^\pi(\state_i,\goal_k) -  V^\pi(\state_j,\goal_k)\big|.$$
Plugging in \cref{eq:apply_value_fn},
$$|V_R^\pi(\state_i) - V_R^\pi(\state_j)| \leq \sum_{k=0}^{|\mathcal{S}|} \alpha_k d^\pi(\state_i,\goal_k; \state_j,\goal_k).$$
\end{proof}

\section{Additional Ablations}
\label{app:ablations}
In this section, we run experiments that modify various aspects of our algorithm to show how different design choices and hyperparameters affect performance.

\subsection{Dimensionality of $\phi$ and $\psi$}
We start by examining how changing the latent dimensionality of the two spaces we learn, denoted by $\phi$ and $\psi$, affect downstream success rate on the \texttt{Drawer} task. In \cref{fig:ablatedim} we find that this hyperparameter does not significantly affect downstream performance of the control task. All other ablations are performed on the \texttt{Button and Drawer}.
\begin{figure}[h]
  \centering
  \includegraphics[width=0.495\linewidth]{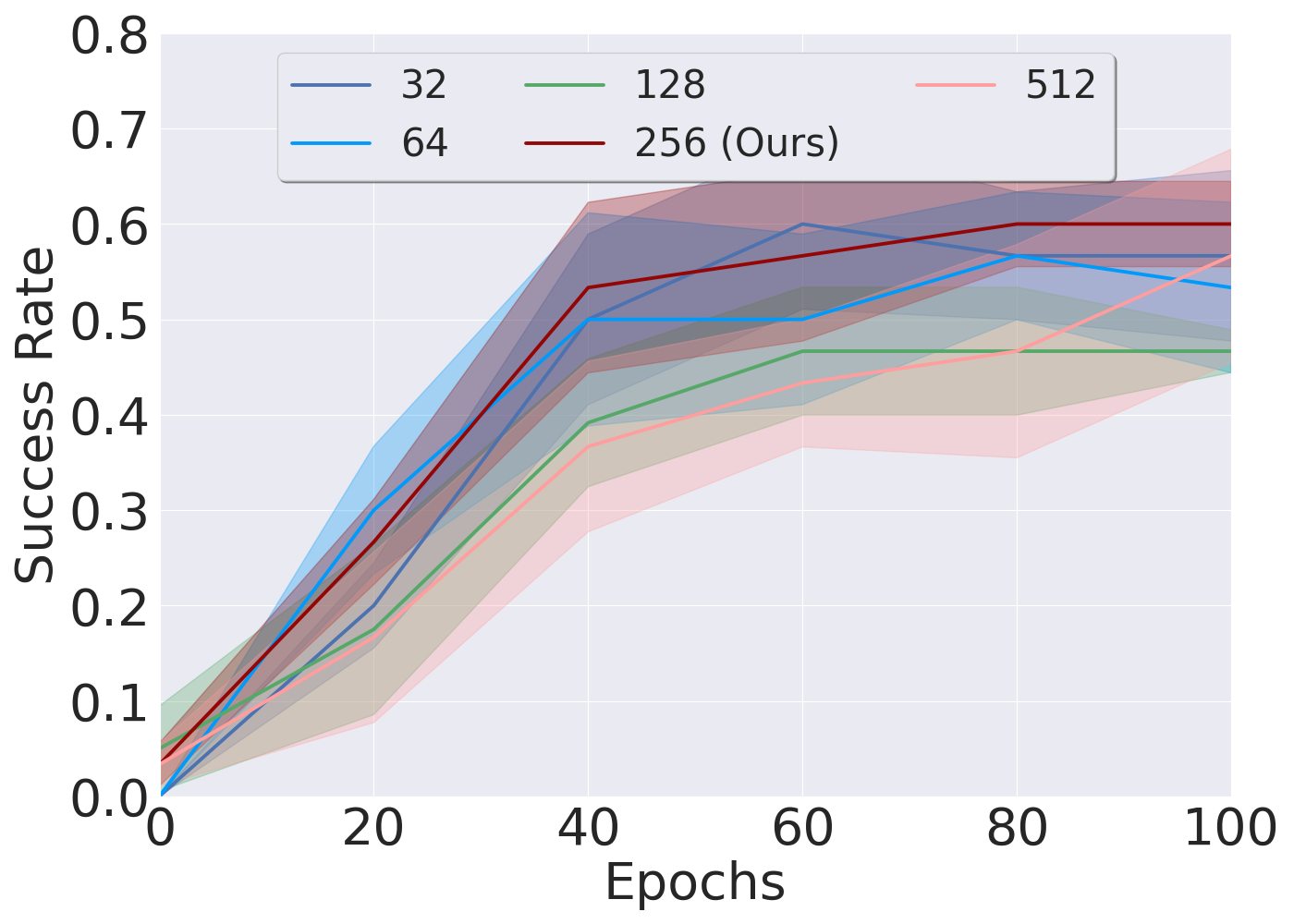}
  \caption{Ablation on different dimensionality for $\phi$ and $\psi$. We see that our tuned hyperparameter of 256 performs best, but the range of values from 32 to 512 generally does not affect success rate significantly. }
  \label{fig:ablatedim}
\end{figure}

\subsection{Grounding $\psi$ with Goal Point $\phi(g,g)$}
As discussed in \cref{sec:method}, the loss for $\psi$ in \cref{eq:psi_obj} which differs from the one proposed in \cref{eq:analogy}. Considering $\phi$ is trained with pairs, $\phi(s,g)$ alone might not properly represent an informative point. Instead, $\phi(s,g) - \phi(g,g)$ will more likely have a well defined distance as it involves a comparison of pairs or tasks. We refer to this as grounding $\psi$ with the goal point. We ablate on this difference in \cref{fig:ablateground} by running GCB where $\psi$'s objective is either \cref{eq:psi_obj} (GCB-G) or \cref{eq:analogy} (GCB-NG). Grounding $\psi$ has a strong performance boost compared to not grounding. 

\begin{figure}[h]
  \centering
  \includegraphics[width=0.495\linewidth]{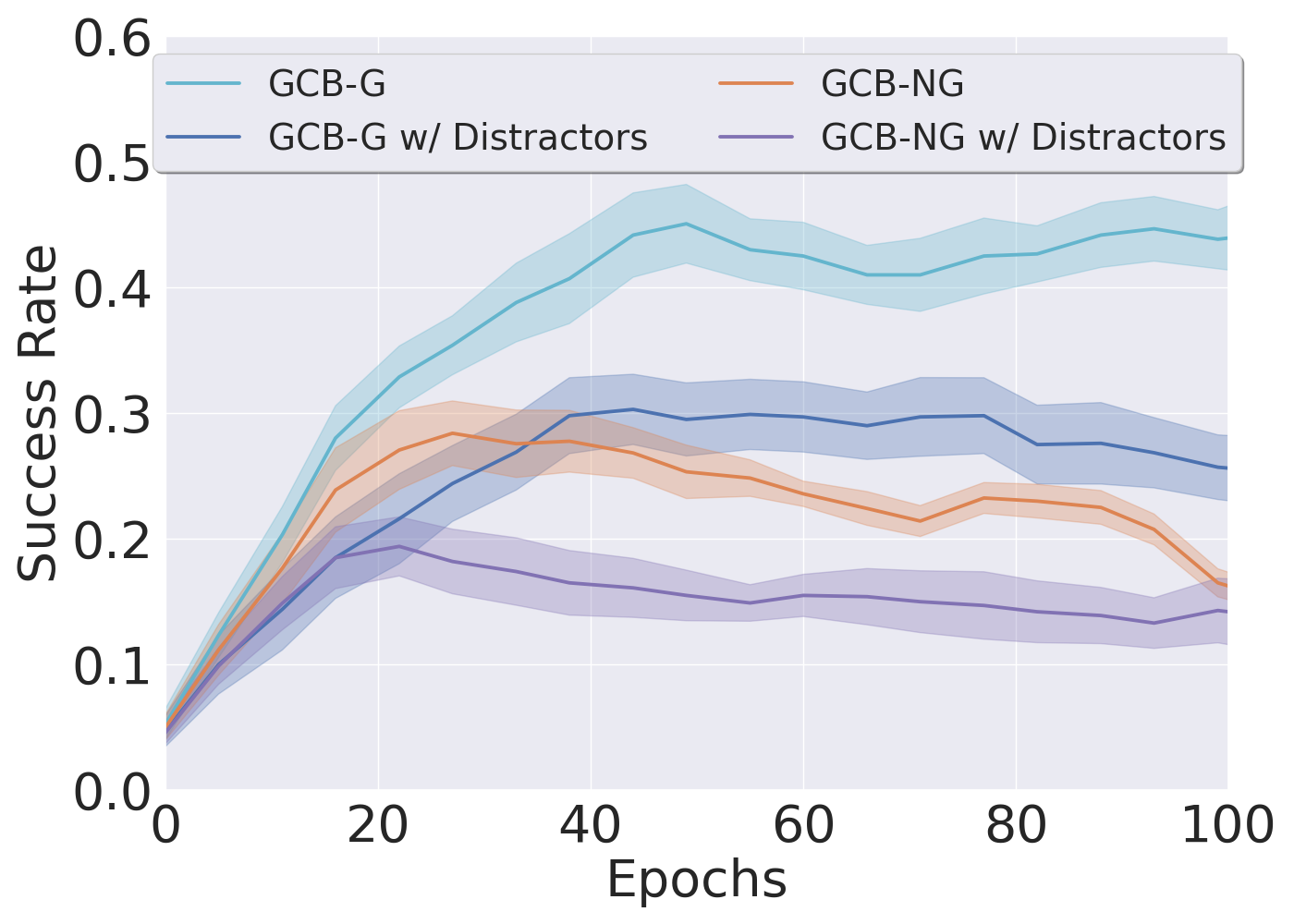}
  \caption{Ablation on different $\psi$ objectives. \ouracronym-G uses the grounded loss and \ouracronym-NG does not. Video Distractors are added for some ablations.}
  \label{fig:ablateground}
\end{figure}

\subsection{Using $\ell_1$ vs $\ell_2$ Norm Objective}
$\phi$'s loss in \cref{eq:phi_loss} uses an $\ell_1$ metric loss on pairs to fit $\phi$'s encoder. In this section we ablate on using $\ell_1$ or $\ell_2$ ($||\phi(s_i, g_i) -  \phi(s_j, g_j)||_2)$) for the first portion of $\phi$'s loss. We report results in \cref{fig:ablatemetric}. There is a large difference between the two losses, $\ell_1$ achieving far better performance. We suspect that $\ell_1$ results in more stable training of $\phi$.
\begin{figure}[h]
  \centering
  \includegraphics[width=0.495\linewidth]{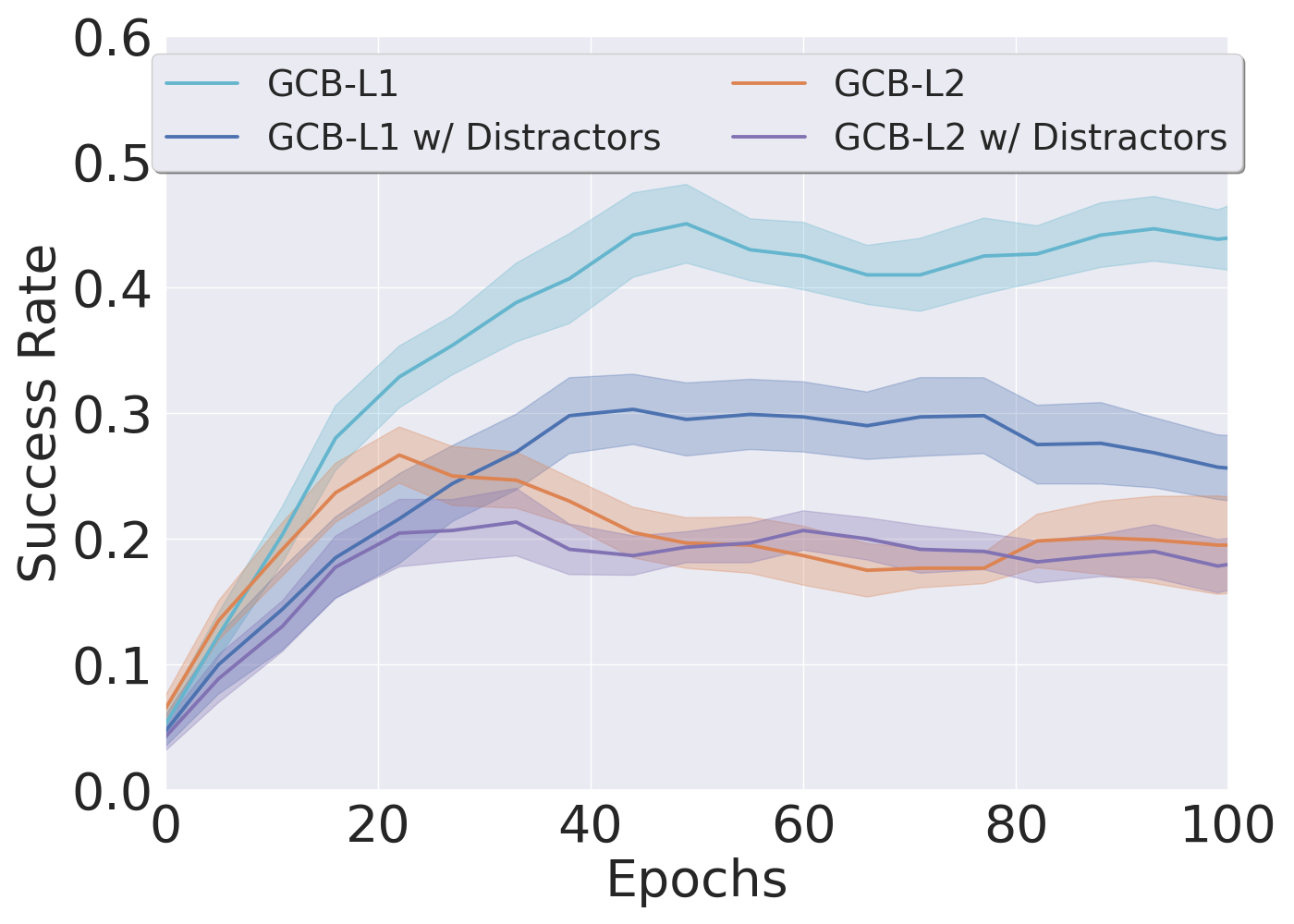}
  \caption{Ablation on different $\phi$ objectives. \ouracronym-$\ell_1$ uses an $\ell_1$ metric loss and \ouracronym-$\ell_2$ uses $\ell_2$ . Video Distractors are added for some ablations.}
  \label{fig:ablatemetric}
\end{figure}

\subsection{Backpropagating RL Objectives through \ouracronym}

All other baselines compared to in \cref{sec:exps} allow critic gradients from IQL's Q function into the encoder. We ablate using critic gradients with \cref{eq:psi_obj} to see if it is helpful. Results are shown in \cref{fig:ablateCG}. Critic gradients seem marginally harmful for downstream performance, so we decided to not use them in \ouracronym.
\begin{figure}[h]
  \centering
  \includegraphics[width=0.495\linewidth]{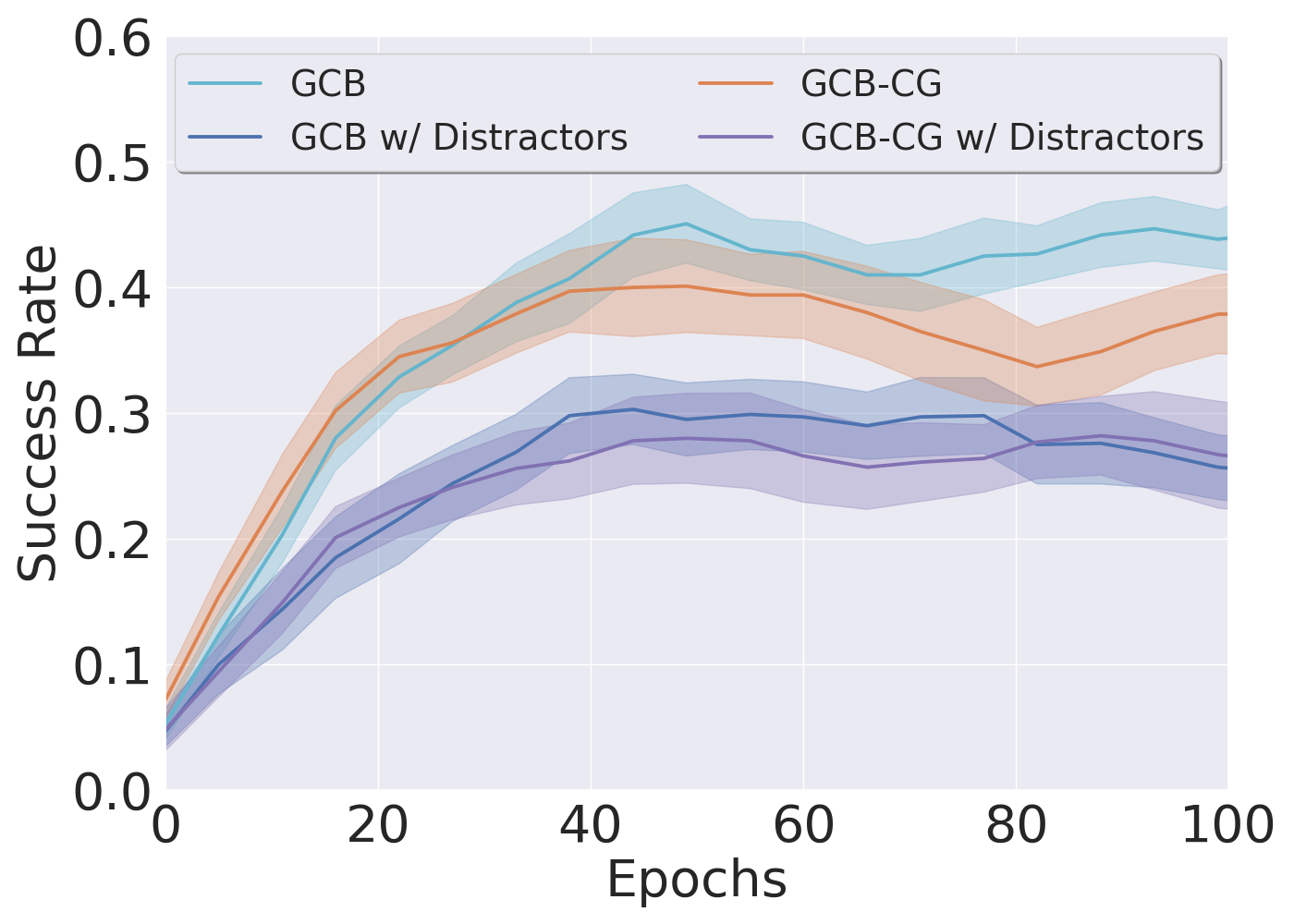}
  \caption{Ablation on different $\phi$ objectives. \ouracronym uses no critic gradients and \ouracronym-CG uses critic gradients. Video Distractors are added for some ablations.}
  \label{fig:ablateCG}
\end{figure}

\subsection{Adding a Reward Decoder}
In the DBC implementation, a reward decoder is used to help stabilize the representation \citep{zhang2021dbc}. Considering that rewards are sparse in our domain, it might be harder to train a decoder. We ablate using and not using a reward decoder. The reward decoder $\mathcal{R}(\phi(s,g), \phi(s',g))$ is an MLP that is trained to predict if a task transition goes to the goal. The reward decoder loss is shown in Equation \ref{eq:r_loss} and the gradient backpropagates through the $\phi$ encoder.

\begin{align}\label{eq:r_loss}
\mathcal{L}_{\mathcal{R}} &= \bigg(\mathcal{R}(\phi(s_i,g_i), \phi(s'_i,g_i)) - r_i \bigg)^2,
\end{align} 

We report results in \cref{fig:ablatedecoder}. The results show that \ouracronym has very little performance difference with and without a reward decoder. We decided to leave it in the implementation used for the main results.
\begin{figure}[h]
  \centering
  \includegraphics[width=0.495\linewidth]{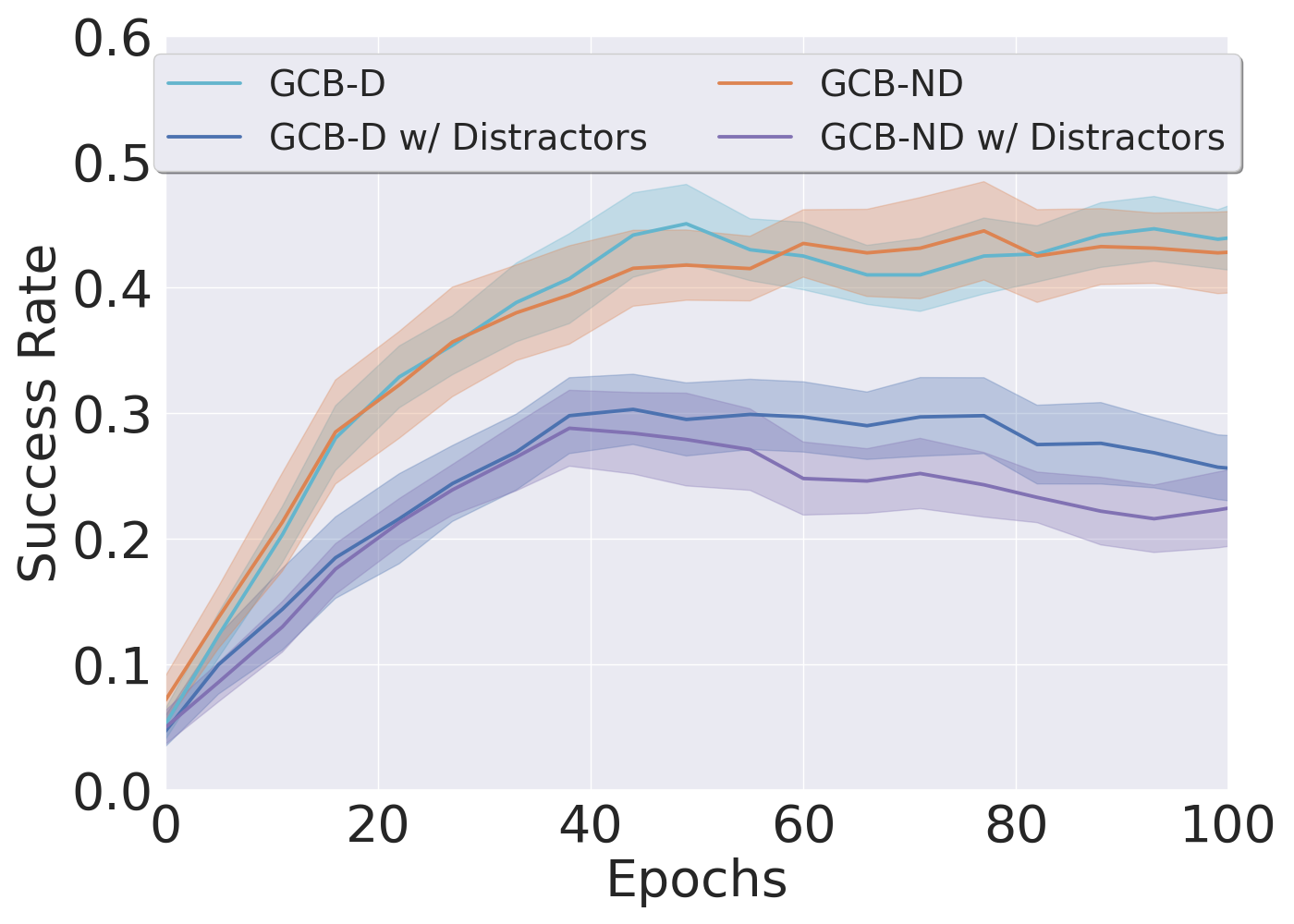}
  \caption{Ablation on different $\phi$ objectives. \ouracronym-D uses a reward decoder and \ouracronym-ND doesn't use a reward decoder. Video Distractors are added for some ablations.}
  \label{fig:ablatedecoder}
\end{figure}

\subsection{Using Dynamics Model vs Next Observation \ouracronym}

In the definitions of bisimulation metrics discussed in Section \ref{sec:prelim} and Section \ref{sec:concepts}, a dynamics model $\mathcal{P}$ was used to define similar next states. There are two possible instantiations for implementing a learned version of this metric: learning a dynamics model that outputs a Gaussian next state distribution and calculating the Wasserstein, as proposed in \citet{zhang2021dbc}, or using sampled interactions, as introduced in \citet{castro2021mico}. While we took the latter approach as our main method, we investigate using a learned dynamics approach here.  An MLP dynamics model $f(\phi(s,g), a)$ could be learned and used in second portion of our loss in Equation \ref{eq:phi_loss}. We define that new $\phi$ loss as follows:

\begin{align}\label{eq:phi_f_loss}
\mathcal{L}_\phi &= \bigg(||\phi(s_i, g_i) - \phi(s_j, g_j)||_1  - ||r_i - r_j||_2  \\ 
&- \gamma ||f(\bar{\phi}(s_i, g_i), a_i) - f(\bar{\phi}(s_j, g_j), a_j) + \bar{\phi}(s_i, g_i) - \bar{\phi}(s_j, g_j)||_2\bigg)^2 \nonumber,
\end{align} 
as well as a dynamics loss for training the model $f$:
\begin{align}\label{eq:f_loss}
\mathcal{L}_f &= \bigg(f(\phi(s,g), a)  + \phi(s',g) - \phi(s,g) \bigg)^2 \nonumber.
\end{align} 
We present ablations of using the standard Equation \ref{eq:phi_loss} loss versus using a dynamics model and \ref{eq:phi_f_loss} in Figure \ref{fig:ablatedynamics}. We find that using next observation results in a more stable learning process compared to using a learned dynamics model, although the dynamics model method is capable of achieving similar peak performance.

\begin{figure}[h]
  \centering
  \includegraphics[width=0.495\linewidth]{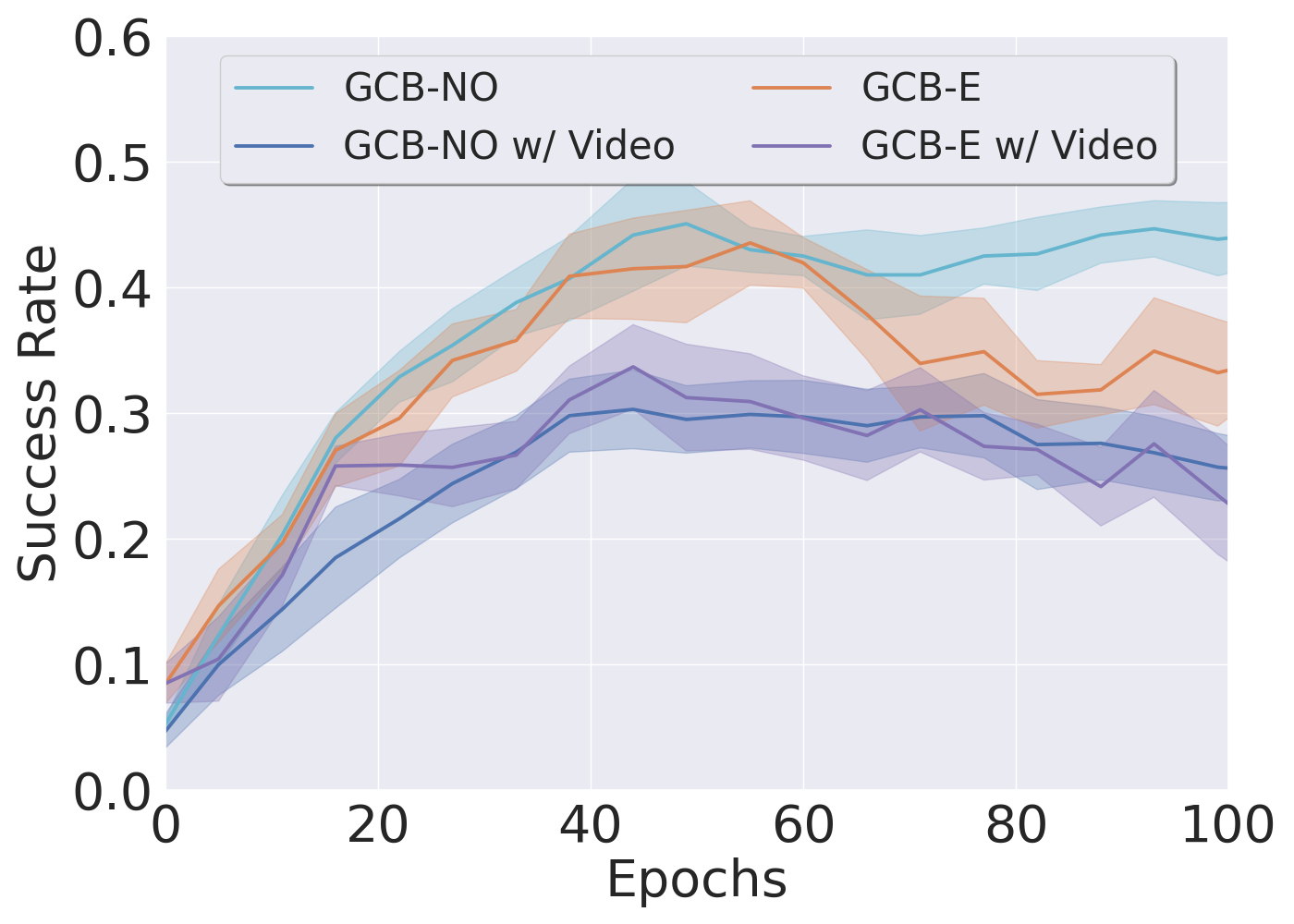}
  \caption{Ablation on different $\phi$ objectives. \ouracronym-NO uses next observation and \ouracronym-E uses the dynamics model based loss. Video Distractors are added for some ablations.}
  \label{fig:ablatedynamics}
\end{figure}

\section{Implementation Details}
\label{app:implementation_details}
We use a similar encoder architecture to \cite{laskin2020rad} except we use a six layer convolutional network. The encoder has kernels of size 3 × 3 with 32 channels for all the convolutional layers and set stride to 1 everywhere, except of the first convolutional layer, which has stride 2, and interpolate with ReLU activations. 
We modify a PyTorch implementation of IQL \citep{kostrikov2021iql} for our offline RL algorithm.
The hyperparameters used for the experiment are in~\cref{tab:hyperparams}. 

\begin{table*}[]
    \centering
    \begin{tabular}{c|c}
      Hyperparameter  & Value \\
      \hline
      Offline Sample Count  &  $50000$ \\
      $\gamma$ & {$0.99$}\\
      Batch Size & {$256$}\\
      $\psi$ or Encoder LR & {$5 \cdot 10^{-4}$}\\
      $\psi$ or Encoder Weight Decay & {$10^{-4}$}\\
      $\phi$ LR & {$10^{-4}$}\\
      $\phi$ Weight Decay & {$ 10^{-3}$}\\
      Actor LR & {$10^{-4}$}\\
      {Actor $\beta$} & {$0.9$}\\
      {Actor logstd Bounds} & {$[-10, 2]$}\\
      {Critic LR} & {$10^{-4}$}\\
      {Critic $\beta$} & {$0.9$}\\
      {Critic $\tau$} & {$0.005$}\\
      {Quantile} & {$0.7$}\\
      {Optimizer} & {Adam}\\
      {Adam $\beta_1$} & {0.9}\\
    \end{tabular}
    \caption{GCB Hyperparameters.}
    \label{tab:hyperparams}
\end{table*}

\section{Environment Details}
\label{app:env_details}

We evaluate our method on two Sawyer manipulation environment domains: \texttt{Drawer} and \texttt{Button and Drawer}. We list details of how we constructed our environments for the PyBullet simulation tasks in Section \ref{sec:exps}. In the following paragraphs we describe specific details about how we collected data, evaluated the agent, and formed analogies.

For each episode of data collection, we first initialize the PyBullet environment with at most four random objects spawned with random poses in addition to a randomly chosen drawer pose and openness such that the drawer handle, when fully open, fits the xy constraints of the environment. We then roll out a scripted policy $\pi^*(s, g) = a^*$ to complete its task until the goal is achieved or until the 75 timestep limit is reached. In the \texttt{Drawer} environment, the task is to open or close the drawer to randomly chosen openness. In the \texttt{Button and Drawer} environment, the task is to either do the same task as in the \texttt{Drawer} environment or to press a button. 

For each episode of evaluation, we randomly initialize the PyBullet environment in the same manner as in data collection. The goal $g$ is collected by rolling out a scripted policy $\pi^*(s, g) = a^*$, rendering the final frame, and marking that as the goal. We then reinitialize the environment with the same object poses for evaluating an agent conditioned on $g$.

The \texttt{Drawer} environment tasks a robotic agent to open or close a drawer to a specific drawer openness. Specifically, it tasks the agent to guide its state image $s$ to a drawer handle and position it as shown in the goal image $g$ using its arm. The environment consists of a drawer with a random pose and at most four random distractor objects.

The \texttt{Button and Drawer} environment tasks a robotic agent to either press a button or guide a drawer handle to a specific drawer openness. The environment consists of a button on top of two drawers stacked on top of one another along with at most four random distractor objects. The top drawer is taken from the \texttt{Drawer} environment, and the bottom drawer is opened or closed by pressing the button.

To form analogous state goal pairs $s_a, g_a$, we first randomly initialize the objects in the the \texttt{Button and Drawer} environment. We then take the state and add noise to the pose of the drawer (translation and rotation). The colors of the drawer, number and type of distractor objects spawned, and the video distractor in the background of $(s_a, g_a)$ will most likely change as well.

In \cref{tb:table1} and \cref{tb:table2}, we provide the environment parameter settings we used on each environment. We also provide some extra screenshots of states and goals. For the table below, let $x^{(1)}_p$ denote the top drawer handle xyz position, $d^{(1)}_p$ denote the desired top drawer handle xyz position, $x^{(2)}_p$ denote the bottom drawer handle xyz position, $d^{(2)}_p$ denote the desired bottom drawer handle xyz position, $x^{(3)}_p$ denote the button z position, and $d^{(3)}_p$ denote the desired button z position.

\begin{figure*}[htb]
 \centering
 \subfloat[\texttt{Drawer} initial state]{\includegraphics[width=0.2\textwidth]{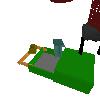}}
 \hspace{0.05\textwidth}
 \subfloat[\texttt{Drawer} final state]{\includegraphics[width=0.2\textwidth]{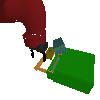}}
 \hspace{0.05\textwidth}
 \subfloat[\texttt{Drawer and} \\ \texttt{Button} initial state]{\includegraphics[width=0.2\textwidth]{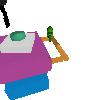}}
 \hspace{0.05\textwidth}
 \subfloat[\texttt{Drawer and} \\ \texttt{Button} final state]{\includegraphics[width=0.2\textwidth]{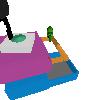}}
 \caption{Sample images from \texttt{Drawer} and \texttt{Drawer and Button} environment.}
\end{figure*}

\begin{table*}[htb]
    \centering
    \begin{tabular}{c|c}
      Parameter & Value \\
      \hline
      Discount  $\gamma$&  $0.99$ \\
      Action Space & {$[-1, 1]^{5}$}\\
      Action Space Description & {[$x$, $y$, $z$, yaw rotation, claw grip]}\\
      DoF &  $5$ \\
      Observation Space (i.e. Image Space) & {$[0,1]^{64\times 64\times 3}$}\\
      Max Time Steps &  $75$ \\
      Offline Dataset Size & 50K transitions \\
      Reward & 
          \begin{tabular}{@{}c@{}}$\mathds{1}_{|x^{(1)}_p-d^{(1)}_p| \leq 0.05}$ \end{tabular} \\
      Number of Distractor Objects & Uniform Random in [0, 4] \\
      Underlying State Space (Not shown to agent) & {$\mathbb{R}^{11}$}\\
      Underlying State Space Description (Not shown to agent) & 
        \begin{tabular}{@{}c@{}}[gripper xyz position, gripper quaternion, gripper tips distance, \\ top drawer xyz position]\end{tabular} \\
    \end{tabular}
    \caption{Parameters used in the \texttt{Drawer} environment.}
    \label{tb:table1}
\end{table*}
\newpage
\begin{table*}[htb]
    \centering
    \begin{tabular}{c|c}
      Parameter & Value \\
      \hline
      Discount  $\gamma$&  $0.99$ \\
      Action Space & {$[-1, 1]^{5}$}\\
      Action Space Description & {[$x$, $y$, $z$, yaw rotation, claw grip]}\\
      DoF &  $5$ \\
      Observation Space (i.e. Image Space) & {$[0,1]^{64\times 64\times 3}$}\\
      Max Time Steps &  $75$ \\
      Offline Dataset Size & 50K transitions \\
      Reward &         
        ${
            \begin{cases} 
              \mathds{1}_{|x^{(1)}_p-d^{(1)}_p| \leq 0.05} & \text{if doing drawer task} \\
              \mathds{1}_{|x^{(2)}_p-d^{(2)}_p| \leq 0.05 \land |x^{(3)}_p-d^{(3)}_p| \leq 0.008} & \text{if doing button task} \\
            \end{cases}
        }$\\
      Number of Distractor Objects & Uniform Random in [0, 4] \\
      Underlying State Space (Not shown to agent) & {$\mathbb{R}^{15}$}\\
      Underlying State Space Description (Not shown to agent) & 
        \begin{tabular}{@{}c@{}}[gripper xyz position, gripper quaternion, gripper tips distance, \\ top drawer xyz position, bottom drawer xyz position, button z position]\end{tabular} \\
    \end{tabular}
    \caption{Parameters used in the \texttt{Drawer and Button} environment.}
    \label{tb:table2}
\end{table*}

\end{document}